\documentclass{article}
\usepackage[utf8]{inputenc} 
\usepackage{amsthm}

\usepackage[T1]{fontenc}    
\usepackage{amsmath,amsfonts,amssymb}
\usepackage{nicefrac}       
\usepackage{microtype}      
\usepackage{bm}
\usepackage{geometry}
\usepackage{multirow}
\usepackage{hyperref}
\usepackage{url}
\usepackage{graphicx}
\usepackage{todonotes}
\usepackage{lmodern}        
\usepackage[T1]{fontenc}    

\usepackage{epstopdf}
\usepackage{algorithm}
\usepackage{algorithmicx}
\usepackage{algpseudocode}
\usepackage[numbers,square,comma,sort&compress]{natbib}
\usepackage{cleveref}
\usepackage{comment}
\usepackage{varwidth}

\usepackage{booktabs} 
\usepackage{soul}  
\usepackage[shortlabels]{enumitem}
\graphicspath{{}{figures/}{plots/}}
\allowdisplaybreaks

\definecolor{OliveGreen}{rgb}{0,0.6,0}

\usepackage{cancel}

\def\lV{\left\lVert }
\def\rV{\right\lVert }
\def\lv{\left\lvert}
\def\rv{\right\rvert}

\def\l{\left\langle}
\def\r{\right\rangle}

\def\H{\mathcal{H} }

\def\G{\mathcal{G} }
\def\I{\mathcal{I} }

\def\Ca{\mathcal{C}}
\def\P{\mathcal{P} }

\def\tr{\mathrm{tr} }

\def\W{\mathcal{W} }
\def\C{\mathbb{C} }

\def\bz{\bm{z} }
\def\bzs{\bm{z}_{\star} }
\def\bx{\bm{x} }

\def\bs{\bm{s}}
\def\bss{\bm{s}_{\star} }

\def\bf{\bm{f} }
\def\bv{\bm{v} }

\def\cO{\mathcal{O} }
\def\BA{{\bm{A}} }
\def\BB{{\bm{B}} }
\def\BL{{\bm{L}} }
\def\BR{{\bm{R}} }
\def\BU{{\bm{U}} }
\def\BV{{\bm{V}} }
\def\BSigma{\bm{\Sigma} }
\def\BSigmas{\bm{\Sigma}_{\star} }
\def\BQ{{\bm{Q}} }
\def\BI{{\bm{I}} }

\def\BX{{\bm{X}} }
\def\BY{{\bm{Y}} }
\def\BO{{\bm{O}} }

\def\BDeltaL{\bm{\Delta}_{\BL}}
\def\BDeltaR{\bm{\Delta}_{\BR}}
\def\BDeltaRk{\bm{\Delta}_{\BR_k}}
\def\BDeltaLk{\bm{\Delta}_{\BL_k}}
\def\BDelta{\bm{\Delta}}
\def\BLn{\bm{L}_{\natural}}
\def\BRn{\bm{R}_{\natural}}
\def\BLs{\bm{L}_{\star}}
\def\BRs{\bm{R}_{\star}}
\def\sigmas{\sigma^{\star}}
\def\SVD{\mathrm{SVD} }
\def\fro{\mathrm{F}}
\def\rank{\mathrm{rank}}
\def\supp{\mathrm{supp}}
\def\Re{\mathrm{Re}}

\DeclareMathOperator*{\minimize}{\mathrm{minimize}}
\DeclareMathOperator*{\subjectto}{\mathrm{subject~to}}

\newtheorem{theorem}{Theorem}

\newtheorem{assumption}{Assumption}
\newtheorem{claim}{Claim}
\newtheorem{lemma}{Lemma}
\newtheorem{proposition}{Proposition}
\newtheorem{corollary}{Corollary}
\newtheorem{remark}{Remark}
\Crefname{assumption}{Assumption}{Assumptions}
\Crefname{example}{Example}{Examples}
\Crefname{remark}{Remark}{Remarks}
\Crefname{claim}{Claim}{Claims}

\begin{document}

\title{Accelerating Ill-conditioned Hankel Matrix Recovery via Structured Newton-like Descent
}

\author{HanQin Cai\thanks{Department of Statistics and Data Science and Department of Computer Science, University of Central Florida, Orlando, FL 32816, USA. (email: \href{mailto:hqcai@ucf.edu}{hqcai@ucf.edu}) }
\and
Longxiu Huang\thanks{Department of Computational Mathematics, Science and Engineering and Department of Mathematics, Michigan State University, East Lansing, MI 48823, USA. (email:  \href{mailto:huangl3@msu.edu}{huangl3@msu.edu}).}
\and
Xiliang Lu\thanks{School of Mathematics and Statistics, and Hubei Key Laboratory of Computational Science, Wuhan University, Wuhan 430072, China.
(email: \href{mailto:xllv.math@whu.edu.cn}{xllv.math@whu.edu.cn})}
\and
 Juntao You\thanks{School of Artificial Intelligence, Wuhan University, National Center for Applied Mathematics in Hubei, and Hubei Key Laboratory of Computational Science,  Wuhan, China.(Corresponding author: \href{mailto:youjuntao@whu.edu.cn}{youjuntao@whu.edu.cn}).} $^{,}$\thanks{Institute for Advanced Study, Shenzhen University, Shenzhen 518000, China.  }
 }

\date{}

\maketitle

\begin{abstract}
This paper studies the robust Hankel recovery problem, which simultaneously removes the sparse outliers and fulfills missing entries from the partial observation. We propose a novel non-convex algorithm, coined Hankel Structured Newton-Like Descent (HSNLD), to tackle the robust Hankel recovery problem. HSNLD is highly efficient with linear convergence, and its convergence rate is independent of the condition number of the underlying Hankel matrix. The recovery guarantee has been established under some mild conditions. Numerical experiments on both synthetic and real datasets show the superior performance of HSNLD against state-of-the-art algorithms. \\

\end{abstract}


\section{Introduction} The problem of low-rank Hankel matrix recovery has widely appeared in magnetic resonance imaging\cite{jin2016general,jacob2020structured}, nuclear magnetic resonance spectroscopy \cite{holland2011fast,qu2015accelerated}, seismic analysis \cite{wang2018hankel,chen2015robust}, system identification \cite{fazel2013hankel,smith2014frequency}, time series forecasting \cite{gillard2018structured,sun2021hankel,chen2024laplacian}, and many other real-world applications. Mathematically, 
a complex-valued Hankel matrix is structured as: 
\begin{align*} \label{eq:Hankel definition}
\bm{X} =
\begin{bmatrix}
x_1 & x_2  & \cdots & x_{n_2}\\
x_2 & x_3  & \cdots & x_{n_2+1}  \\
\vdots &  \vdots & \cdots  &\vdots\\
x_{n_1} & x_{n_1+1}  & \cdots   &x_{n}
\end{bmatrix}\in\C^{n_1\times n_2}.
\end{align*}
Since each skew-diagonal of a Hankel matrix has the same values, it contains only $n = n_1+n_2-1$ distinct values in total. In practice, keeping track of a memory-efficient reweighed vector 
\begin{equation*}
\bz = [x_1,\sqrt{2}x_2,\cdots,\sqrt{\varsigma_i}x_i,\cdots,\sqrt{2}x_{n-1},x_n]^\top\in\C^n
\end{equation*}
is sufficient to represent a large-scale Hankel matrix, where $\varsigma_i$ is the number of entries on the $i$-th antidiagonal of an $n_1\times n_2$ matrix. Note that $\|\bz\|_2=\|\bm{X}\|_\fro$ which is convenient for formulating the Hankel recovery problem in its equivalent vector form.
We also introduce a reweighed Hankel mapping $\G:\C^n\rightarrow\C^{n_1\times n_2}$ and a reweighting operator $\W:\C^n\rightarrow\C^n$ such that
\begin{equation} \label{eq:Hankelnotation}
    \G\bz=\bm{X} \quad\text{ and }\quad \W\bz = [x_1,x_2,\cdots,x_i,\cdots,x_{n-1},x_n]^\top.
\end{equation}
Note $\G^{*}\G=\I$ where $(\cdot)^*$ denotes the adjoint operator and $\I$ is the identity operator.

In the applications of the low-rank Hankel matrix, we face three major challenges: (i) \textit{missing data}, (ii) \textit{sparse outliers}, and (iii) \textit{ill condition}. To address the first two challenges, we can formulate the Hankel matrix recovery problem:
\begin{equation} \label{eq:Hankel_objective_matrix}
\begin{split}
    \minimize_{\BX,\BO\in\C^{n_1\times n_2}} &~ \frac{1}{2p} \l\P_{\bm{H}_\Omega} \BY -\P_{\bm{H}_\Omega}(\BX+\BO), \BY -(\BX+\BO)\r \cr
    \subjectto & ~ \BX \textnormal{ is rank-$r$ Hankel matrix and } \BO \textnormal{ is sparse,} 
\end{split}
\end{equation}
where $p$ is the observation rate, $\BO$ is the sparse outlier matrix,  $\P_{\bm{H}_\Omega}: \C^{n_1\times n_2}\rightarrow\C^{n_1\times n_2}$ is the matrix observation operator, and $\P_{\bm{H}_\Omega} \BY$ is the partially observed data matrix with sparse corruptions. However, as discussed in the literature \cite{cai2023structured}, randomly positioned, non-structured missing entries and outliers can be easily recovered through the redundant skew-diagonal in the Hankel matrix. One can also Hankelize $\P_{\bm{H}_\Omega} \BY$ by averaging its skew-diagonal elements as $\BX$ is Hankel structured. Thus, we are only interested in the more challenging recovery problem with Hankel structured missing entries and outliers, i.e., some skew-diagonals are entirely missed or corrupted. Under this setting, all parties of \eqref{eq:Hankel_objective_matrix} are Hankel structured including the observation pattern $\bm{H}_\Omega$, thus we can rewrite the Hankel recovery problem in the equivalent vector form: 
\begin{equation} \label{eq:Hankel_objective_vector} 
    \begin{split}
        \minimize_{\bz,\bs\in\C^n}\ & ~\frac{1}{2p}\l\Pi_{\Omega}\bf-\Pi_{\Omega}(\bz+\bs),\bf-(\bz+\bs) \r \cr
        \subjectto & ~ \rank\left(\G \bz\right)=r,  \quad \Pi_{\Omega}\bs \textnormal{ is }\alpha \textnormal{-sparse},
    \end{split}
\end{equation}
where $\Pi_\Omega: \C^n\rightarrow\C^n$ is the observation operator on $\C^n$ that corresponds to the Hankel structured sampling pattern, $\bf=\G^*\BY$ and $\bs=\G^*\BO$ are observed data and outlier matrices in their equivalent vector form, and $\alpha$ is a sparsity parameter defined later in \Cref{amp:sparsity}. Since only the observed outliers interfere with the recovery,  we put the sparsity constraint on $\Pi_\Omega \bs$ instead of the entire $\bs$.

Ill condition is one major challenge that remains unsolved for existing Hankel recovery literature, as the existing algorithms often struggle with convergence speed (detailed in \Cref{sec:related}). 
Unfortunately, ill-conditioned Hankel matrices often appear in many applications, specifically, when the frequency of a spectrally sparse signal is inseparable, the corresponding Hankel matrix has a large condition number \cite{li2020robust, li2021stable}. 
An example is Direction of Arrival (DOA) estimation which can be modeled as a low-rank Hankel recovery problem \cite{yang2018sparse}. When multiple sources arrive from close angles, the corresponding Hankel matrix has a large condition number, see \Cref{fig:DOA_explain} for illustration. 
DOA and many other applications of Hankel recovery, such as traffic anomaly detection \cite{chen2024laplacian} and seismic monitoring \cite{chen2015robust}, are very time-sensitive, thus developing highly efficient solvers is an eager task. This paper aims to study a provable non-convex algorithm for solving \eqref{eq:Hankel_objective_vector} that keeps high computational efficiency even if the condition number of the underlying Hankel matrix is large.

\begin{figure}
\vspace{-0.1in}
    \centering
    \includegraphics[width=0.66\linewidth]{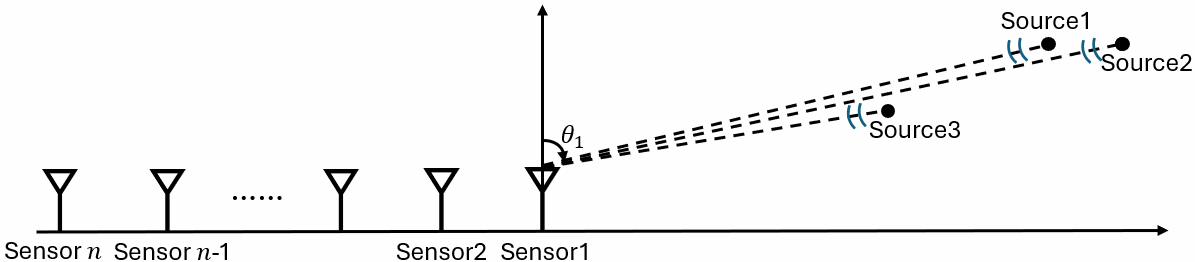}
    \hfill
    \includegraphics[width=0.32\linewidth]{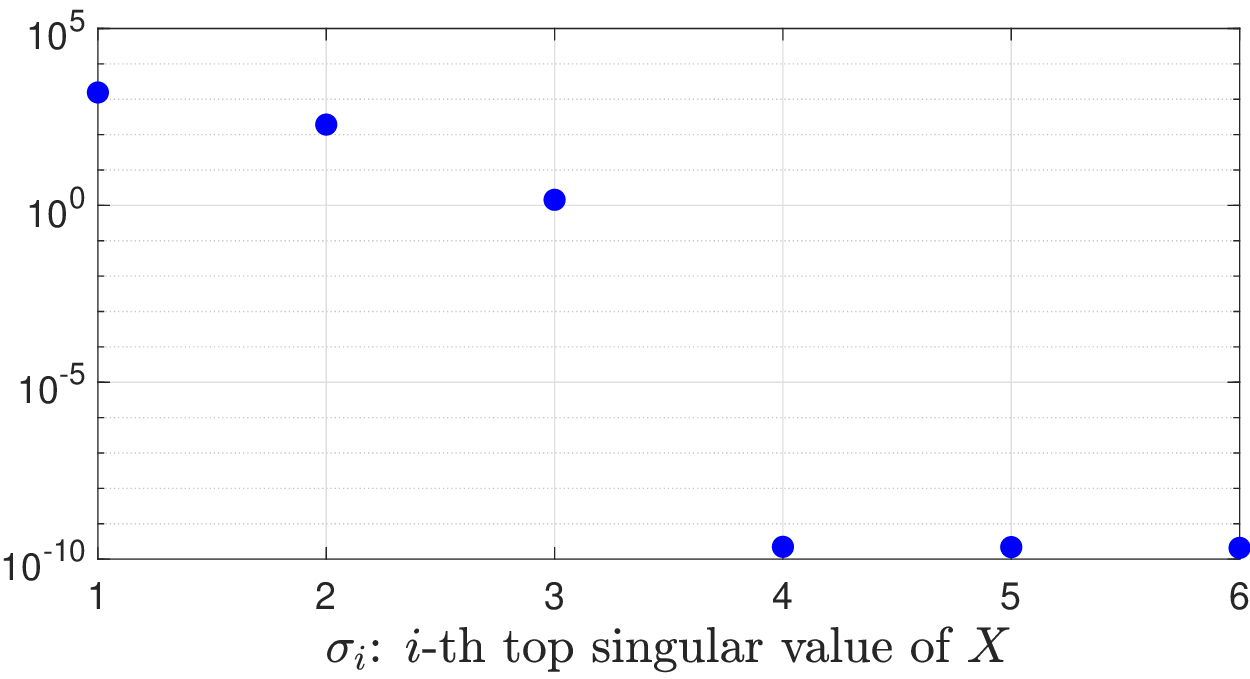}
    \caption{Left: DOA estimation for $r=3$ far-field sources from the directions $\theta= 87^{\circ}, 87.1^{\circ}, 87.3^{\circ}$, signal collected by a Uniform Linear Array with $n=2^{12}$ sensors. Right: Top six singular values of the corresponding square Hankel matrix. For a rank-3 approximation, the condition number $\kappa=\sigma_1/\sigma_3\approx 1083$. See \Cref{subsec:DOA} for more details of the DOA experiment.}
    \label{fig:DOA_explain}
    \vspace{-0.2in}
\end{figure}


\subsection{Assumptions} Denote the vector-form groundtruth as $\bzs$ and $\bss$ for low-rank Hankel and sparse outlier respectively. 
To address the Hankel recovery problem \eqref{eq:Hankel_objective_vector}, we introduce the following standard assumptions on data sampling, Hankel matrix incoherence, and outlier sparsity.

\begin{assumption}[Uniform sampling with replacement] \label{amp:Replacement}
Let the observation index set $\Omega=\{a_k\}_{k=1}^m$ be sampled uniformly  with replacement from $[n]$. The observation rate is   $p:=|\Omega|/n=m/n$. 
\end{assumption}
In this context, the sampling operator is defined as $\Pi_\Omega\bv=\sum_{i\in\Omega} v_i \bm{e}_i$ and allows repeated samples, thus $\l \Pi_\Omega\bv,\bv \r\neq \|\Pi_\Omega\bv\|_2^2$.


\begin{assumption}[$\mu$-incoherence] \label{amp:incoherence}
Let $\BU_{\star}\BSigmas\BV_{\star}^{*}$ be the compact singular value decomposition (SVD) of the underlying rank-$r$ Hankel matrix $\G \bzs\in\C^{n_1\times n_2}$. Then $\G \bzs$ is said $\mu$-incoherent if there exists a positive constant $\mu$ such that
\begin{equation*}
\| \BU_{\star} \|_{2,\infty}\le \sqrt{\mu c_s r/n} \quad \textnormal{ and } \quad \| \BV_{\star} \|_{2,\infty}\le \sqrt{\mu c_s r/n},
\end{equation*}
where $c_s=\max\{{n}/{n_1},{n}/{n_2}\}$ measures how square the Hankel matrix is.
\end{assumption}

\begin{assumption}[$\alpha p$-sparsity] \label{amp:sparsity}
Given a sampling model (i.e., \Cref{amp:Replacement}), the observed outliers are at most $\alpha p$-sparse. That is,  $  \|\Pi_\Omega \bss\|_0\leq \alpha |\Omega| = \alpha p n$.
\end{assumption}

\Cref{amp:Replacement,amp:incoherence,amp:sparsity} have been widely used in matrix completion \cite{candes2009exact,recht2011simpler,wei2020guarantees,cai2023ccs}, robust principal component analysis \cite{candes2011robust,netrapalli2014non,cai2019accaltproj,cai2021lrpca}, and low-rank Hankel recovery literature \cite{chen2014robust,cai2018spectral,cai2019fast,cai2021asap}, ensuring that \eqref{eq:Hankel_objective_vector} is a well-posed problem \cite{chandrasekaran2011rank}. 
 \Cref{amp:sparsity} states that only a $\alpha$ proportion of observed data are corrupted, without specifying the randomness of the corruption pattern. This condition holds with high probability as long as $\bss$ is $\frac{\alpha}{2}$-sparse \cite{cai2021rcur}. Moreover, the sparsity of $\bss$ (resp.~$\Pi_\Omega\bss$) implies the sparsity of $\G\bss$ (resp.~$\G\Pi_\Omega\bss$) in each of its rows and columns, provided $n_1\approx n_2$. 

\subsection{Notation} \label{sec:notation}
Throughout the paper, we use regular lowercase letters for scalars (e.g., $n$), bold lowercase letters for vectors (e.g., $\bv$), bold capital letters for matrices (e.g., $\bm{M}$), and calligraphic letters for operators (e.g., $\G$). $[n]$ denotes the set $\{1,2,\cdots,n\}$ for any positive integer $n$. Given a vector, $\|\bm{v}\|_0$, $\|\bm{v}\|_2$, and $\|\bm{v}\|_\infty$ denote its $\ell_0$-, $\ell_2$-, and $\ell_\infty$-norms respectively. $|\bm{v}|$ is the entrywise absolute value of $\bm{v}$. Given a matrix $\bm{M}$, $\sigma_i(\bm{M})$ denotes the $i$-th singular value, $\|\bm{M}\|_{2,\infty}$ denotes the largest row-wise $\ell_2$-norm, $\|\bm{M}\|_\infty$ denotes the largest entry-wise magnitude, $\|\bm{M}\|_2$ and $\|\bm{M}\|_\fro$ denote the spectral and Frobenius norms respectively. $\|\G\|$ denotes the operator norm of a linear operator. For all, $\overline{(\cdot)}$, $(\cdot)^\top$, $(\cdot)^*$, and $\langle\cdot,\cdot\rangle$ denote conjugate, transpose, conjugate transpose, and inner product respectively. 
Moreover, for ease of notation, we denote $\sigmas_i$ for the $i$-th singular value of the underlying Hankel matrix $\G\bzs$ and $\kappa=\sigmas_1/\sigmas_r$ for its condition number. Finally, $a\lesssim \cO(b)$ means there exists an absolute numerical constant $c > 0$ such that $a$
can be upper bounded by $cb$..

\subsection{Related work and contributions} \label{sec:related}
The low-rank Hankel recovery problem has been extensively studied in the past decade. A naive approach is to directly apply robust matrix recovery algorithms \cite{chen2013low,yi2016fast,cherapanamjeri2017nearly,cai2021ircur,cai2021rcur,hamm2022RieCUR,cai2024rtcur,cai2024rccs} on the large $n_1\times n_2$ Hankel matrix which suffers from high computational complexities (i.e., $\cO(n^2r)$ to $\cO(n^3)$ per iteration) and inferior recoverability since the Hankel structure is ignored. 
Lately, many dedicated Hankel recovery algorithms have been proposed. Robust-EMaC \cite{chen2014robust} proposes a convex relaxation for the non-convex problem \eqref{eq:Hankel_objective_matrix} and solves it with semidefinite programming which is still computationally and memory too expensive. 
\cite{cai2018spectral,cai2019fast,tong2021accelerating} propose three algorithms with linear convergence that solve the non-convex Hankel completion problem in $\cO(r^2n +rn\log n)$ flops per iteration; however, none of them handles outliers and \cite{tong2021accelerating} also lacks theoretical guarantee for Hankel recovery problem. 
SAP \cite{zhang2019correction} proposes to effectively solve \eqref{eq:Hankel_objective_vector} via alternating projections. The recovery guarantee is established with $m \gtrsim\cO(c_s^2\mu^2r^3\log^2 n)$ and $\alpha\lesssim \cO(1/(c_s \mu r))$. 
The overall computational complexity of SAP is $\cO(r^2n\log n\log(1/\varepsilon))$ with a large hidden constant due to a truncated SVD used in every iteration. The accelerated version of SAP, namely ASAP \cite{cai2021asap}, improves the complexity to $\cO((r^2n +rn\log n)\log(1/\varepsilon))$ flops. However, ASAP studies only the fully observed cases. 
HSGD \cite{cai2023structured} solves \eqref{eq:Hankel_objective_vector} via projected gradient descent in $\cO(r^2n +rn\log n)$ flops per iteration, with a small constant. Although HSGD has a guaranteed linear convergence,  its iteration complexity is $\cO(\kappa \log(1/\varepsilon))$, which is problematic for heavily ill-conditioned problems. The recovery guarantee of HSGD requires $m \gtrsim \cO(\max\{c_s^2\mu^2 r^2\log n, c_s\mu \kappa^3 r^2\log n\})$ and $ \alpha \lesssim \cO( 1/(\max\{c_s \mu\kappa^{3/2}r^{3/2}, c_s \mu r\kappa^2\}))$.  

In this paper, our main contributions are three-fold:
\begin{enumerate}[(i),leftmargin=0.45in]
    \item We propose a highly efficient algorithm for heavily ill-conditioned robust Hankel recovery problems, coined Hankel Structured Newton-Like Descent (HSNLD). Its overall computational cost is $\cO((r^2 n+rn\log n)\log(1/\varepsilon))$ flops which beats all state-of-the-art algorithms. Especially when compared with HSGD, our convergence rate is independent of the condition number $\kappa$; when compared with SAP, our complexity is $\cO(\min\{\log(n),r\})$ lower and the hidden constant is much smaller since no iterative truncated SVD involved.
    \item The recovery guarantee, with linear convergence, has been established for HSNLD under some mild conditions. The sample complexity and outlier tolerance are $m\gtrsim \cO(c_s^2\mu^2 r^2\kappa^4\log n)$ and $\alpha\lesssim \cO(1/( c_s\mu r\kappa^2))$ respectively. Compared to HSGD, the balance regularization is omitted from the objective, a new distance measurement is employed and a series of new technical lemmas has been developed accordingly, which will benefit future studies in the field.
    \item We verify the empirical advantages of HSNLD with synthetic datasets, as well as DOA and NMR signals. HSNLD achieves the best recoverability and computational efficiency in all ill-conditioned test problems. It gains more advantages when the condition numbers become even worse. 
\end{enumerate}

\section{Proposed method} 
Hankel Structured Gradient Descent (HSGD) \cite{cai2023structured} is a state-of-the-art non-convex algorithm for Hankel matrix recovery problems. It has shown superior performance for well-conditioned Hankel matrices, both theoretically and empirically. However, when the underlying Hankel matrix is ill-conditioned, the convergence of HSGD is significantly slower in theory. In particular, HSGD requires $\cO(\kappa \log\varepsilon^{-1})$ iterations to find a $\varepsilon$-solution. This phenomenon is empirically verified in \Cref{fig:HSGD_kappas}, where we observe that HSGD suffers from the larger condition number of the underlying Hankel matrix. To overcome this major weakness of HSGD, we propose a novel non-convex algorithm inspired by recent developments on Newton-like Gradient Descent algorithms \cite{tong2021accelerating}. The proposed algorithm, dubbed Hankel Structured Newton-like Descent (HSNLD), is highly efficient even for ill-conditioned problems. 

\begin{figure}
    \centering
    \includegraphics[width=.36\linewidth]{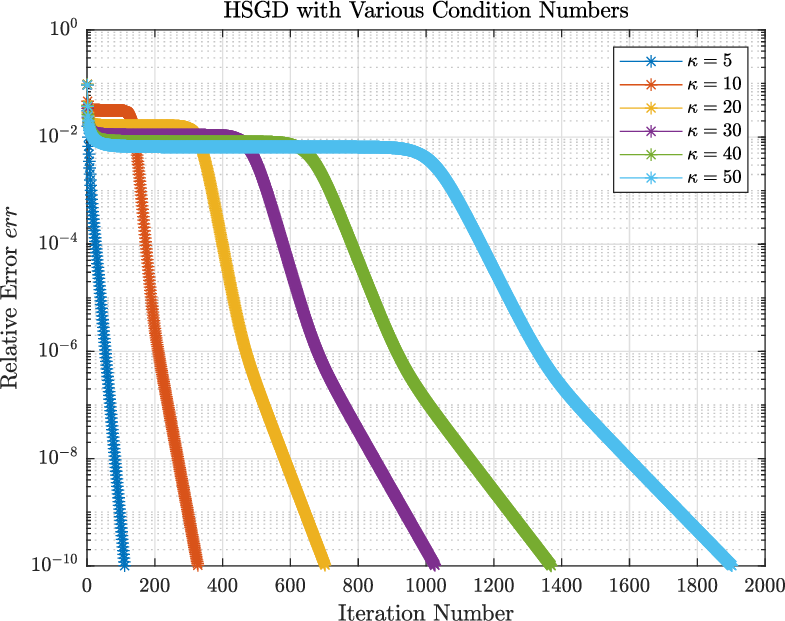}
    \vspace{-0.12in}
    \caption{The convergence performance of HSGD \cite{cai2023structured} with different condition numbers $\kappa$.}
    \label{fig:HSGD_kappas}
    \vspace{-0.2in}
\end{figure}

To recover the low-rank Hankel matrix from missing data and sparse outliers, we aim at a similar objective function as used in \cite{cai2023structured}: 
\begin{equation} \label{eq:def-L}
    \resizebox{0.9\textwidth}{!}{$\ell\left(\BL,\BR;\bs\right)=\frac{1}{2p}\l \Pi_{\Omega}\left(\G^{*}\left(\BL \BR^{*}\right)+\bs-\bm{f}\right),\G^{*}\left(\BL \BR^{*}\right)+\bs-\bm{f}\r +\frac{1}{2}\lV\left(\I-\G\G^{*}\right)\left(\BL \BR^{*}\right)\rV_\fro^2$},
\end{equation}
where $\bz=\G^*\BX\in\C^n$ and we parameterize 
$\BX\in\C^{n_1\times n_2}$ as a product of  $\BL\in\C^{n_1\times r}$ and $\BR\in\C^{n_2\times r}$, thus the low-rank constraint is implicitly encoded. Moreover, a regularization term $\lV\left(\I-\G\G^{*}\right)\left(\BL \BR^{*}\right)\rV_\fro^2$ is added to enforce the Hankel structure of $\BL\BR^*$. Targeting \eqref{eq:def-L}, the proposed HSNLD (summarized in \Cref{algo:HankelscaledGD}) will be discussed step by step in the following paragraphs. For the sake of presentation, the iterative steps will be discussed first, and then initialization.

\begin{remark}
    Compared to the objective posted in \cite{cai2023structured}, our objective \eqref{eq:def-L} omits the balance regularization term $\|\BL^*\BL-\BR^*\BR\|_\fro^2$. The balance regularization is not empirically needed for the original HSGD algorithm but is required for its convergence analysis. In this paper, we drop the balance regularization for the proposed HSNLD with an improved convergence guarantee.
\end{remark}

\begin{algorithm}[h]
\caption{\textbf{H}ankel \textbf{S}tructured \textbf{N}ewton-\textbf{L}ike \textbf{D}escent (HSNLD)} \label{algo:HankelscaledGD}
\begin{algorithmic}[1]
\State \textbf{Input:} 
$\Pi_\Omega\bf$: partial observation on the corrupted Hankel matrix in reweighted vector form; $r$: the rank of underlying Hankel matrix; $p$: observation rate; $\alpha$: outlier density; $\{\gamma_k\}$: parameters for sparsification operator; $\eta$: step size; $K$: maximum number of iterations. 
\State \textbf{Initialization:} Let $\BL_{0}, \BR_{0}$ be produced by \eqref{ini:L0R0}.
\For{ $k=0,1,\ldots,K-1$ }
    \State $\bz_{k+1}=\G^{*}(\BL_{k}\BR_{k}^*)$ \hfill 
    \State $\bs_{k+1}=\Gamma_{\gamma_k\alpha p}(\Pi_{\Omega}(\bf-\bz_{k+1}))$
    \State $\begin{bmatrix} \BL_{k+1}\\ \BR_{k+1} \end{bmatrix}
    =\P_{\Ca}\begin{bmatrix}\BL_{k}-\eta\nabla_{\BL} \ell(\BL_{k},\BR_{k};\bs_{k+1})\left(\BR^{*}_k\BR_k\right)^{-1} \\
       \BR_{k}-\eta\nabla_{\BR} \ell(\BL_{k},\BR_{k};\bs_{k+1})\left(\BL^*_k\BL_k\right)^{-1}\end{bmatrix}$ 
\EndFor
\State \textbf{Output:} $\bz_{\mathrm{output}}=\G^{*}(\BL_{K}\BR_K^{*})$
: recovered Hankel matrix in vector form.
\end{algorithmic}
\end{algorithm}

\vspace{0.03in}
\noindent\textbf{Iterative updates on $\bs$.} For the sake of computational efficiency, one should avoid forming the full-size Hankel matrix and process the calculation in the corresponding vector form. Hence, we first update $\bz$ with the current low-rank components, i.e.,  $ \bz_{k+1}=\G^{*}(\BL_{k}\BR_{k}^{*})$. 
By the definition of $\G^{*}$, the $t$-th entry of $\bz$ can be written as
\begin{equation*}
    [\G^{*}(\BL\BR^{*})]_t= \sum_{j=1}^r [\G^{*}(\BL_{:,j}\BR_{:,j}^*)]_t = \sum_{j=1}^r \frac{1}{\sqrt{\varsigma_t}}\sum_{i_1+i_2=t+1} L_{i_1,j} \overline{R}_{i_2,j}.
\end{equation*}
Thus, $z_{k+1}$ 
can be obtained via $r$ fast convolutions, with a computational cost in $\cO(rn \log n)$ flops.

To detect the outliers, we compute the residue $\Pi_{\Omega}(\bf-\bz_{k+1})$ over the observation and project it onto the space of sparse vectors via a sparsification operator. That is, $\bs_{k+1}=\Gamma_{\gamma_k\alpha p}\Pi_{\Omega}\big(\bf-\bz_{k+1}\big)$,
where $\Gamma_{\gamma_k\alpha p}:\C^n\rightarrow\C^n$ is the vector sparsification operator that keeps the $\gamma_k\alpha p n$ largest-in-magnitude entries unchanged and vanishes the rest to $0$. 
Therein, $\gamma_k > 1$ is a parameter that tolerates slight overestimation on the outlier sparsity as the true $\alpha$ is usually unavailable to the user, and it also provides redundant robustness in outlier detection. We observe that taking $\gamma_k\rightarrow 1$ as $k\rightarrow \infty$ provides balanced performance in robustness and efficiency; however, $\gamma_k$ should always be larger than $1$. With off-the-shelf fast partial sorting algorithm, the computational cost of $s_{k+1}$ 
is $\cO(n\log n)$ flops or better.

\vspace{0.03in}
\noindent\textbf{Iterative updates on $\BL$ and $\BR$.} 
The state-of-the-art HSGD \cite{cai2023structured} updates the low-rank Hankel components via iterative gradient descent. As discussed, its convergence is hindered by ill-conditioning.  
Inspired by recent work 
\cite{tong2021accelerating,cai2021lrpca,giampouras2024guarantees}, we utilize an additional Newton-like preconditioner in the iterative updates of $\BL$ and $\BR$ to properly address the ill-conditioned Hankel matrices. We have 
\begin{align} \label{eq:update_Hankel}
\begin{bmatrix} \BL_{k+1}\\ \BR_{k+1} \end{bmatrix}
    =\P_{\Ca}\begin{bmatrix}\BL_{k}-\eta\nabla_{\BL} \ell(\BL_{k},\BR_{k};\bs_{k+1})\left(\BR^{*}_k\BR_k\right)^{-1} \\
       \BR_{k}-\eta\nabla_{\BR} \ell(\BL_{k},\BR_{k};\bs_{k+1})\left(\BL^*_k\BL_k\right)^{-1}\end{bmatrix}
\end{align}
where $\eta>0$ is the step size and $\P_{\Ca}$ is a projection operator that enforces the weak incoherence condition, i.e., $\max\{\|\BL\BR^{*}\|_{2,\infty}, \|\BR\BL^{*}\|_{2,\infty}\}\le   C$,
on the low-rank Hankel matrix. In particular, modified from \cite{tong2021accelerating} we define

\begin{align*}
\P_{\Ca}\begin{bmatrix}  \widetilde{\BL}\\ \widetilde{\BR}\end{bmatrix}
    =\begin{bmatrix}\BL \\ \BR\end{bmatrix},~\text{where}~\BL_{i,:}&=\min\left(1,{C}/{\big\|\widetilde{\BL}_{i,:}(\widetilde{\BR}^*\widetilde{\BR})^{\frac{1}{2}}\big\|_2}\right)\widetilde{\BL}_{i,:},\quad \forall i\in[n_1],\cr
    \BR_{j,:}&=\min\left(1,{C}/{\big\|\widetilde{\BR}_{j,:}(\widetilde{\BL}^*\widetilde{\BL})^{\frac{1}{2}}\big\|_2}\right)\widetilde{\BR}_{j,:},\quad \forall j\in[n_2],
\end{align*}
for some numerical constant $C>0$. The choice of $C$ will be specified in \Cref{sec:Theoretical results} and the local non-expansiveness of $\P_{\Ca}$ is verified in \Cref{PC:Non-expansiveness}.

To efficiently update $\BL$ and $\BR$, we notice that
\begin{align*}
     \nabla_\BL \ell(\BL,\BR;\bs)(\BR^{*}\BR)^{-1}     
    &=\resizebox{0.68\textwidth}{!}{$\left[\G\big(\frac{1}{p}\Pi_\Omega\left(\G^{*}\left(\BL\BR^{*}\right)+\bs-\bf\right)\big)\BR+\left(\BI-\G\G^{*}\right)(\BL\BR^{*})\BR\right](\BR^{*}\BR)^{-1}$}  \cr
    &= \G\big(\frac{1}{p}\Pi_\Omega\left(\G^{*}\left(\BL\BR^{*}\right)+\bs-\bf\right)-\G^{*}(\BL\BR^{*})\big)\BR(\BR^{*}\BR)^{-1}  + \BL.
\end{align*}
One can compute $\bz:=\G^{*}(\BL\BR^{*})$ via $r$ fast convolutions which cost $\cO(rn \log n)$ flops. Computing $\frac{1}{p}\Pi_\Omega\left(\bz+\bs-\bf\right)-\bz$ costs $\cO(n)$ flops or better. Since $\G(\cdot)\BR$ can be viewed as $r$ convolutions, it costs $\cO(rn \log n)$ flops. Computing $\BR(\BR^{*}\BR)^{-1}$ costs $\cO(r^2n)$ flops, so does the projection $\P_{\mathcal{C}}$. Thus, the total computational cost of updating $\BL$ (and similarly for $\BR$) is $\cO(rn\log n+r^2 n)$ 

\vspace{0.03in}
\noindent\textbf{Initialization.}
A tight initialization is important to the success of HSNLD. We adopt the modified Hankel spectral initialization from \cite{cai2023structured}:
\begin{equation}\label{ini:L0R0}
\begin{split}
\bs_{0}=\W^{-1} \Gamma_{\alpha p}\left(\W\Pi_\Omega\bf\right), 
\BU_{0}\BSigma_{0}\BV_{0}^{*} = \SVD_r(\frac{1}{p}\G(\Pi_{\Omega}\bf-\bs_{0})), 
\begin{bmatrix} \BL_{0}\\ \BR_{0} \end{bmatrix}
    &=\P_{\Ca}\begin{bmatrix}\BU_{0}\BSigma_{0}^{1/2}\\ \BV_{0}\BSigma_{0}^{1/2}\end{bmatrix}.
\end{split}
\end{equation}
This initialization has shown promising performance and a low computational complexity of $\cO(rn\log n)$. Note that the truncated SVD of a Hankel matrix costs as low as $\cO(rn\log n)$ flops since all matrix-vector multiplications therein can be viewed as convolutions.

Therefore, the overall computational complexity of the proposed algorithm is $\cO(rn\log n+r^2n)$ which is tied to the current state-of-the-art HSGD \cite{cai2023structured}. In addition, HSNLD does not require forming  the entire Hankel matrix, and only tracking the $n$ distinct elements in the corresponding vector form is sufficient. Thus, the proposed HSNLD is computationally and memory efficient.

\section{Theoretical results} \label{sec:Theoretical results}
We present the theoretical results of the proposed HSNLD in this section. We will first define some convenient notations and metrics for the presentation. Denote $\BLs:=\BU_{\star}\BSigmas^{1/2}$ and $\BRs:=\BV_{\star}\BSigmas^{1/2}$ where $\BU_{\star}\BSigmas\BV_{\star}^{*}=\G\bzs$ is the compact SVD of the underlying Hankel matrix.
Let $(\BL_{k},\BR_{k})$ denote the update at the $k$-th iteration in \Cref{algo:HankelscaledGD}. Define the distance between $(\BL_{k},\BR_{k})$ and $(\BLs,\BRs)$ to be 
\begin{equation}\label{def-error}
  d_k^2:=\mathrm{dist}^2(\BL_{k},\BR_{k};\BLs,\BRs) 
:=\inf_{\BQ\in\mathrm{GL}(r,\C)}\big\| \left(\BL_{k}\BQ-\BLs\right)\BSigmas^{1/2}  \big\|_\fro^2+\big\| \left(\BR_{k}\BQ^{-*}-\BRs\right)\BSigmas^{1/2} \big\|_\fro^2.
\end{equation}
Here, $\mathrm{GL}(r,\C)$ denotes the general linear group of degree $r$ over $\C$ and the invertible matrix $\BQ$ is the best alignment between $(\BL_{k}, \BR_{k})$ and $(\BLs,\BRs)$. {\cite[Lemma~22]{tong2021accelerating}} has shown that $\BQ$ exists if $d_k$ is sufficiently small. In that case, the infimum in \eqref{def-error} can be replaced by the minimum, and the recovery error $\lV\BL_{k}\BR_{k}^*-\G\bzs\rV_{\fro}$ can be bounded by $d_k$ as stated in the following proposition.
\begin{proposition}\label{prop:dist}
 For any $\varepsilon_0\in (0,1)$, provided $d_k< \varepsilon_0\sigmas_r$, then it holds that
\begin{equation*}
\lV\BL_{k}\BR_{k}^*-\G\bzs\rV_{\fro}\le \sqrt{2}(1+\varepsilon_0/2) d_k.
\end{equation*}
\end{proposition}
\begin{proof}
The proof of this proposition is deferred to \Cref{subsec:proofdist}.
\end{proof}
The following \Cref{thm:Initialization} guarantees that our initialization \eqref{ini:L0R0} will start HSNLD sufficiently close to the ground truth, under some mild conditions. Thus, it is sufficient to bound $d_k$ for the recovery guarantee. In addition, our initial guess also satisfies the weak incoherence condition, which will be used and inherited iteratively in the convergence analysis. For ease of presentation, we denote the basin of attraction as: $\mathcal{E}(\delta,C):=\{(\BL,\BR)~|~\mathrm{dist}(\BL,\BR;\BLs,\BRs)\le \delta
~~\textnormal{and}~~
\max\{\lV\BL\BR^{*}\rV_{2,\infty},\lV\BR\BL^{*}\rV_{2,\infty}\}\le C\}$. 


\begin{theorem}\label{thm:Initialization}
For any $\varepsilon_0\in (0,1)$, let $C=c\sqrt{c_s\mu r/n}\,\sigmas_1$ with $c\ge 1+\varepsilon_0$, and suppose \Cref{amp:Replacement,amp:incoherence,amp:sparsity} hold with 
$m\geq c_0\varepsilon_0^{-2} c_s\mu \kappa^2 r^2\log n$ and $\alpha\le \varepsilon_0/(50c_s \mu r^{1.5} \kappa)$, where $c_0$ is some universal constant.
Then, the initialization step \eqref{ini:L0R0} satisfies $(\BL_0,\BR_0)\in \mathcal{E}(\varepsilon_0\sigmas_{r},C)$,
with probability at least $1-2n^{-2}$.  In the case of full observation, probability rises to $1$.
\end{theorem}
\begin{proof}
The proof of this theorem is deferred to \Cref{subsec:proofini}.
\end{proof}

\subsection{Full observation}
We start the analysis with the simpler full observation case, i.e., $\Omega=[n]$. \Cref{thm:convergence_full} presents the local linear convergence of HSNLD.

\begin{theorem}\label{thm:convergence_full}
Suppose \Cref{amp:incoherence,amp:sparsity} hold with $\alpha\lesssim \cO(1/( c_s\mu r\kappa^2))$. Choose parameters $C=c\sqrt{c_s\mu r/n}\sigmas_1$ with $c\ge 1.01$, $\gamma_k\in [1 + 1/b_0,2]$ with some fixed $b_0\geq 1$, and $\eta\in (0,1]$. Under full observation, for any $(\BL_{k},\BR_{k})\in \mathcal{E}(0.01\sigmas_r,C)$, it holds that 
\begin{equation*}
    d_{k+1}\le (1-0.6\eta)d_k
\qquad
    \textnormal{and} 
\qquad
(\BL_{k+1},\BR_{k+1})\in \mathcal{E}(0.01\sigmas_r,C).
\end{equation*}
\end{theorem}
\begin{proof}
The proof of this theorem is deferred to \Cref{subsec:prooflocalcon_full}.
\end{proof}

Combining \Cref{prop:dist,thm:Initialization,thm:convergence_full}, we have the recovery guarantee, with a linear convergence, for the proposed HSNLD under full observation.

\begin{corollary}[Recovery guarantee with full observation] \label{thm:main theorem_full}
Suppose \Cref{amp:incoherence,amp:sparsity} hold with $\alpha\lesssim \cO(1/( c_s\mu r\kappa^2))$. Choose parameters $C=c\sqrt{c_s\mu r/n}\sigmas_1$ with $c\ge 1.01$, $\gamma_k\in [1 + 1/b_0,2]$ with some fixed $b_0\geq 1$, and $\eta\in (0,1]$. Then, under full observation, in $K=\cO(\log(1/\varepsilon))$ iterations, the outputs of \Cref{algo:HankelscaledGD} satisfy 
 \begin{equation*}
 \lV\BL_{K}\BR_{K}^*-\G\bzs\rV_{\fro}
 \leq 0.02 (1-0.6\eta)^K \sigmas_r\leq\varepsilon.
\end{equation*}
\end{corollary}


\subsection{Partial observation}
Now we will process the partial observation cases. One major challenge is the dependency among the entries on the antidiagonals of the Hankel matrix, which causes difficulty in uniformly bounding the subsampling error (see \Cref{projectionerr2}). Now, we present the local convergence under partial observation.

\begin{theorem}\label{thm:convergence}
Suppose \Cref{amp:Replacement,amp:incoherence,amp:sparsity} hold with $m\gtrsim \cO(c_s^2\mu^2 r^2\kappa^4\log n)$ and  $\alpha\lesssim \cO(1/( c_s\mu r\kappa^2))$. Choose the parameters $C=c\sqrt{c_s\mu r/n}\sigmas_1$ with $c\ge 1.01$, $\gamma_k\in [1 + 1/b_0,2]$ with some fixed $b_0\geq 1$, and $\eta\in (0,0.6]$. For any fixed $(\BL_{k},\BR_{k})\in \mathcal{E}(0.01\sigmas_r,C)$,
with probability at least $1-8n^{-2}$ the following results hold $ d_{k+1}\le (1-0.6\eta)d_k
~
    \textnormal{and} 
 ~
(\BL_{k+1},\BR_{k+1})\in \mathcal{E}(0.01\sigmas_r,C)$. 
\end{theorem}
\begin{proof}
The proof of this theorem is deferred to \Cref{proofthm:convergence_partial}.
\end{proof}

Combining \Cref{prop:dist,thm:Initialization,thm:convergence}, we have the recovery guarantee,
with a linear convergence,  for the proposed HSNLD under partial observation.

\begin{corollary}[Recovery guarantee with partial observation] 
Suppose \Cref{amp:Replacement,amp:incoherence,amp:sparsity} hold with $m\gtrsim \cO(c_s^2\mu^2 r^2\kappa^4\log n)$ and $\alpha\lesssim \cO(1/( c_s\mu r\kappa^2))$. Let the observation set $\Omega$ be resampled independently every iteration. Choose parameters $C=c\sqrt{c_s\mu r/n}\sigmas_1$ with $c\ge 1.01$, $\gamma_k\in [1 + 1/b_0,2]$ with some fixed $b_0\geq 1$, and $\eta\in (0,0.6]$. Then, in $K=\cO(\log(1/\varepsilon))$ iterations, the outputs of \Cref{algo:HankelscaledGD} satisfy  $\lV\BL_{K}\BR_{K}^*-\G\bzs\rV_{\fro}
\leq 0.02 (1-0.6\eta)^K \sigmas_r\leq\varepsilon$
 with probability at least $1-\cO(n^{-1})$, provided $\log(1/\varepsilon)\ll n$. 
\end{corollary}

\begin{remark}
    Compared to the state-of-the-art HSGD which runs $\cO(\kappa \log(1/\varepsilon))$ iterations for a $\varepsilon$-optimizer, the convergence rate of HSNLD is independent of the condition number $\kappa$. This makes HSNLD a powerful tool for heavily ill-conditioned Hankel applications, e.g., spectrally sparse signals with inseparable frequencies \cite{liao2016music}.
\end{remark}






\section{Numerical experiments}
In this section, we will empirically demonstrate the performance of the proposed HSNLD with synthetic and real datasets. We compare our algorithm with HSGD \cite{cai2023structured}, SAP \cite{zhang2019correction}, and RobustEMaC \cite{chen2014robust}. Parameters are set according to the original paper and further hand-tuned for best performance when applicable. Particularly, both HSNLD and HSGD use the same outlier sparsification parameters: $\gamma_k= 1.05+0.45\times 0.95^{k}$ as suggested in \cite{cai2023structured}, and the step size $\eta =0.5$. Moreover, we estimate the incoherence parameter $\mu$ via one-step Cadzow method \cite{cadzow1988signal} for all algorithms, and set the projection parameter $C$ according to \Cref{thm:convergence_full} for HSNLD. The resampling is not implemented for HSNLD as it is just a proof technique and is not needed in practice. All results were conducted on Matlab with Intel i9-12950HX and 64GB RAM. 
The code is available at \url{https://github.com/caesarcai/HSNLD}.

\subsection{Synthetic dataset} 
We generate an ill-conditioned Hankal matrix in its vector form $\bz\in\C^n$ with $z_t=\sqrt{\varsigma_t}\sum^r_{j=1} a_j e^{2\pi\imath f_j t},$ where $\imath$ denotes the imaginary unit, $\{f_j\}\in\mathbb{R}$ is a randomly selected set of normalized off-grid frequencies, and $\{a_j\}\in\mathbb{R}$ is a set of $r$ evenly spaced numbers over $[1/\kappa,1]$. Thus, $\G\bz\in\C^{n_1\times n_2}$ is exactly rank-$r$ with condition number $\kappa$. Then, $m$ observations are sampled uniformly without replacement from $\bz$, and therein $\alpha m$ entries are randomly chosen for corruption.  We add complex outliers to the corrupted entries whose real and imaginary parts are drawn uniformly over the intervals $[-10\mathbb{E}(|\mathrm{Re}(z_t)|),10\mathbb{E}(|\mathrm{Re}(z_t)|)]$ and $[-10\mathbb{E}(|\mathrm{Im}(z_t)|),10\mathbb{E}(|\mathrm{Im}(z_t)|)]$.

\vspace{0.03in}
\noindent\textbf{Phase transition.} 
We demonstrate the recoverability and robustness of HSNLD via empirical phase transition under various settings. When an algorithm stops at $\|\G\Pi_{\Omega}\bz_k+\G\Pi_{\Omega}\bs_k-\G\Pi_{\Omega}\bf \|_\fro/\|\G\Pi_{\Omega}\bf\|_\fro\leq 10^{-5}$, we considered it a successful recovery if $\|\G\bz_k-\G\bzs\|_\fro/\|\G\bzs\|_\fro\leq10^{-3}$. For every problem setting, each algorithm runs for $20$ trials and reports the result as a pixel on the phase transition, where a white pixel means all $20$ trials were recovered and a black pixel means all $20$ trials were failed. 
The phase transition results are reported in \Cref{fig:sample_rate_vs_outlier,fig:rank_vs_sample_rate,fig:rank_vs_outlier} where one can see HSNLD delivers the best recoverability and robustness for ill-conditioned problems. Note that we only run the phase transition with a reasonably large $\kappa=10$ since the very ill-conditioned cases, such as $\kappa=2,000$, will take HSGD a forbidding long runtime.

\begin{figure}[t]
    \centering
    \includegraphics[width = 0.24\linewidth]{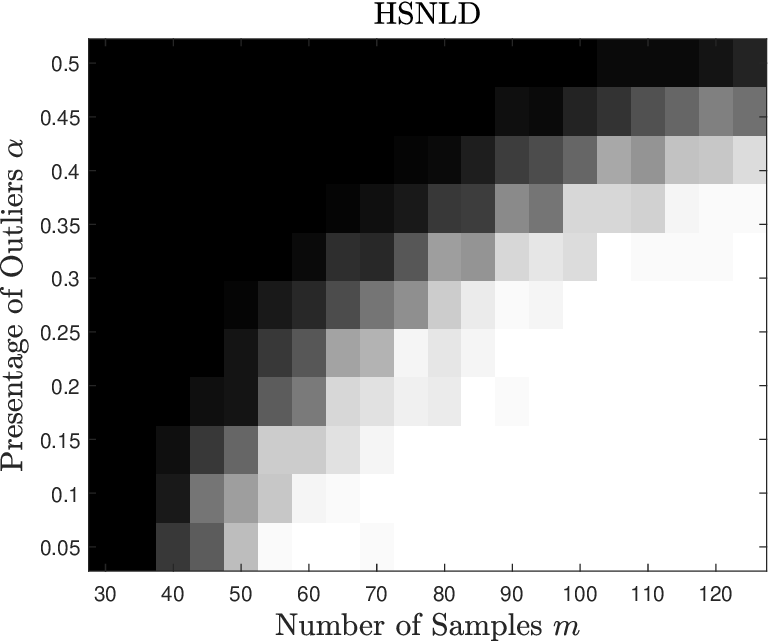}
    \hfill
    \includegraphics[width = 0.24\linewidth]{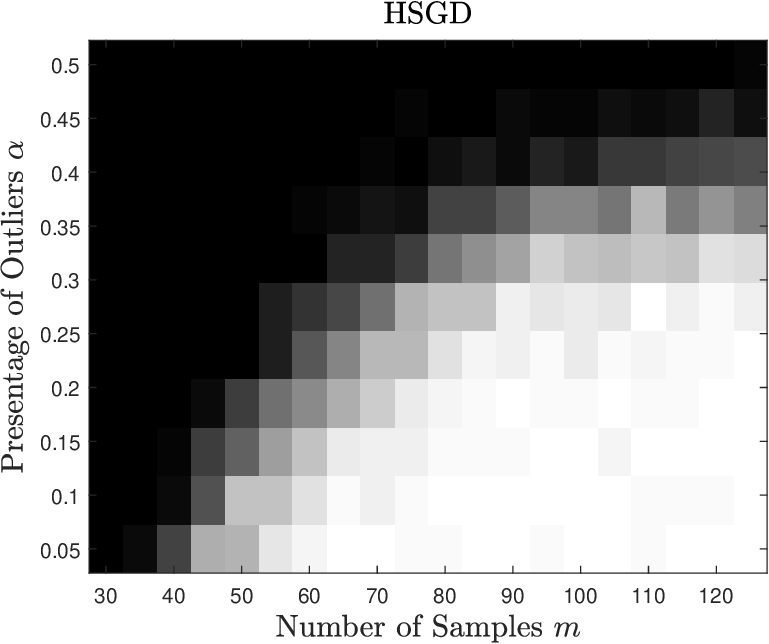}
    \hfill
    \includegraphics[width = 0.24\linewidth]{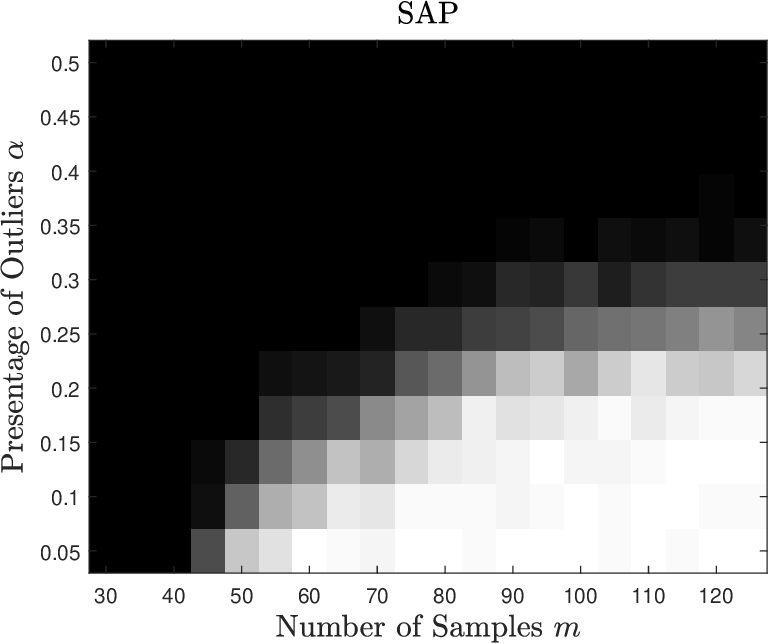}
    \hfill
    \includegraphics[width = 0.24\linewidth]{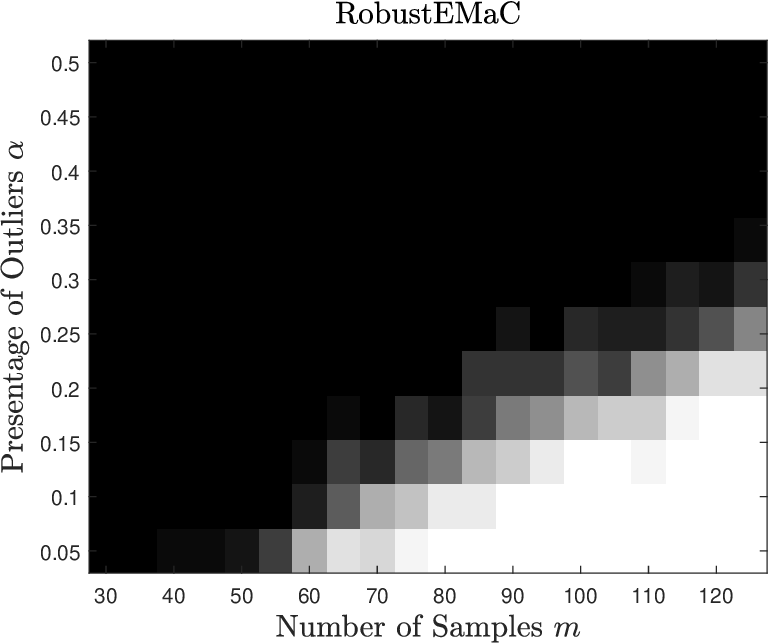}
    \vspace{-0.1in}
    \caption{Phase transition: Number of samples \textit{vs.}~rate of outliers. Problem dimension $n=125$, condition number $\kappa=10$, and rank $r=10$ for every problem.}
    \label{fig:sample_rate_vs_outlier}
\end{figure}

\begin{figure}[t]
    \centering
    \includegraphics[width = 0.24\linewidth]{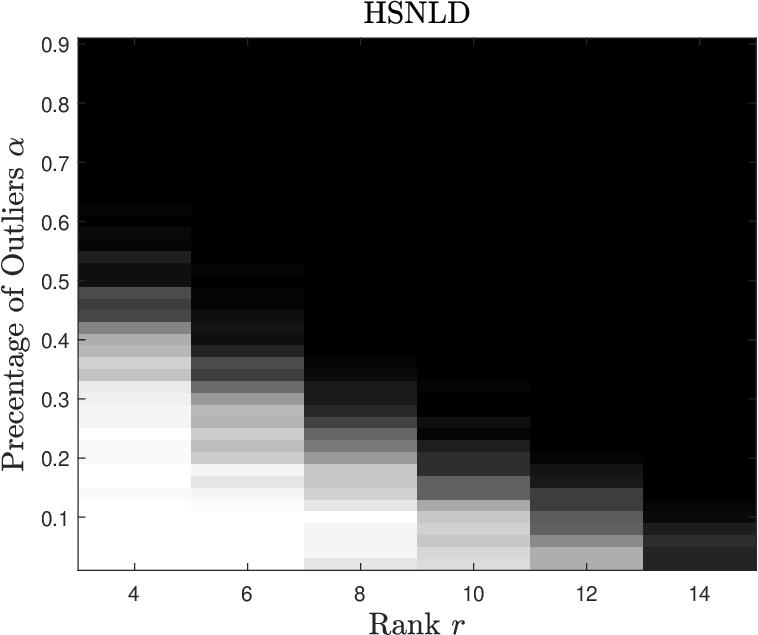}
    \hfill
    \includegraphics[width = 0.24\linewidth]{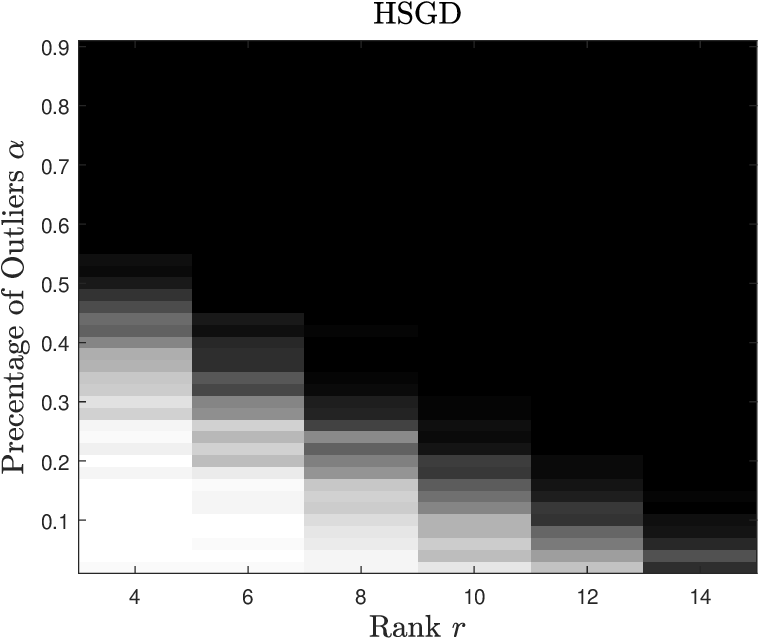}
    \hfill
    \includegraphics[width = 0.24\linewidth]{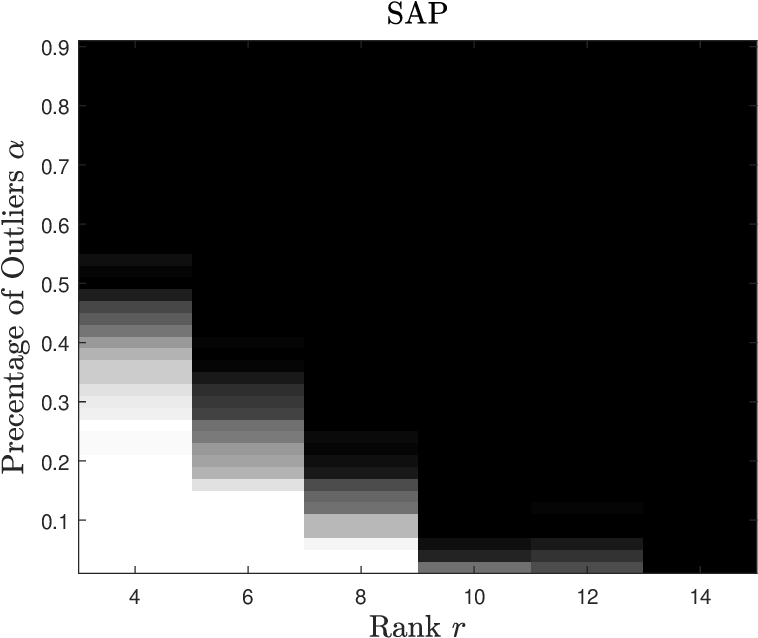}
    \hfill
    \includegraphics[width = 0.24\linewidth]{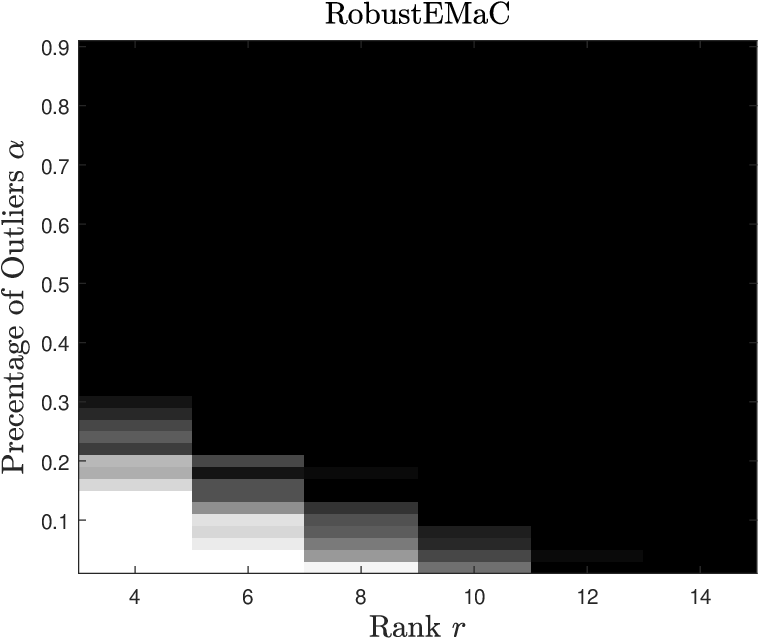}
    \vspace{-0.1in}
    \caption{Phase transition: Rank \textit{vs.}~rate of outliers. Problem dimension $n=125$, condition number $\kappa=10$, and $m=50$ entries are sampled in every problem.}
    \label{fig:rank_vs_outlier}
\end{figure}

\begin{figure}[t]
    \centering
    \includegraphics[width = 0.24\linewidth]{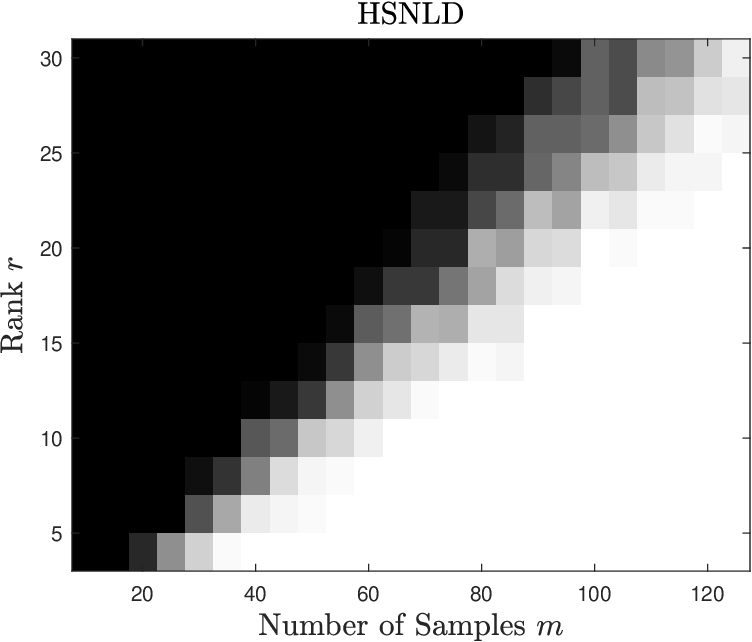}
    \hfill
     \includegraphics[width = 0.24\linewidth]{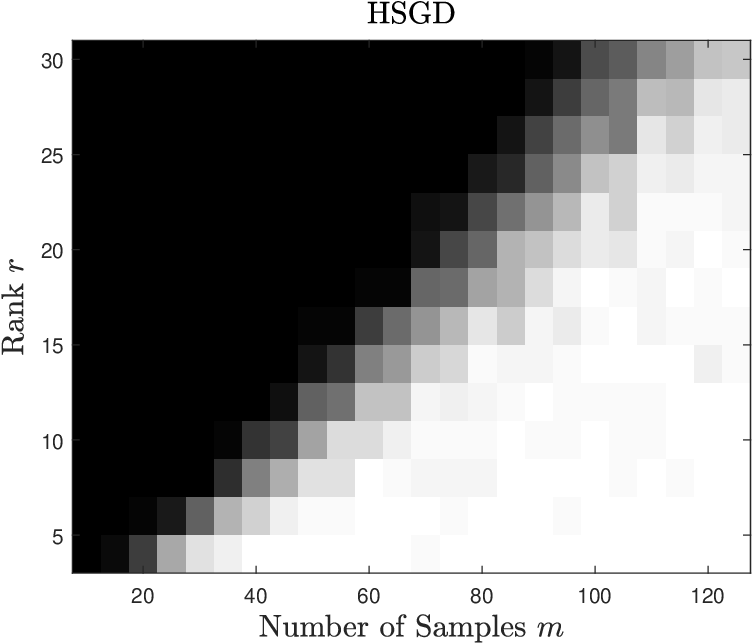}
    \hfill
    \includegraphics[width = 0.24\linewidth]{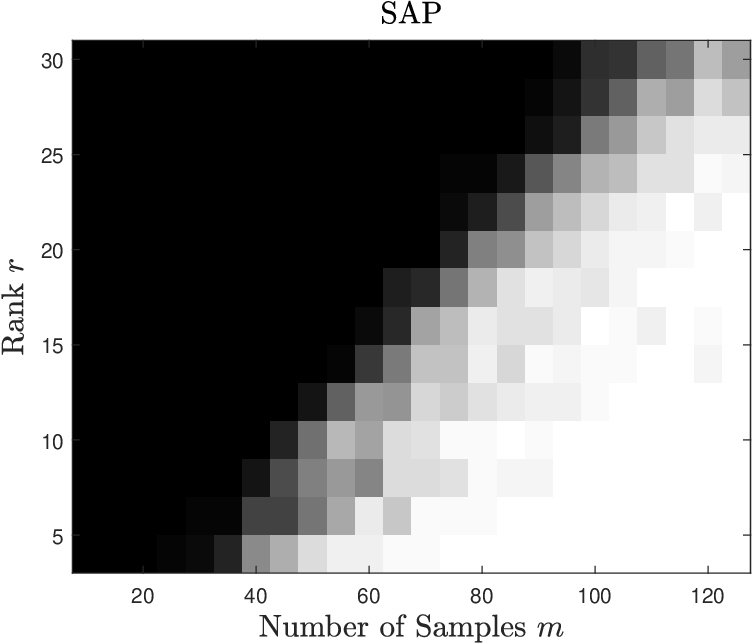}
    \hfill
    \includegraphics[width = 0.24\linewidth]{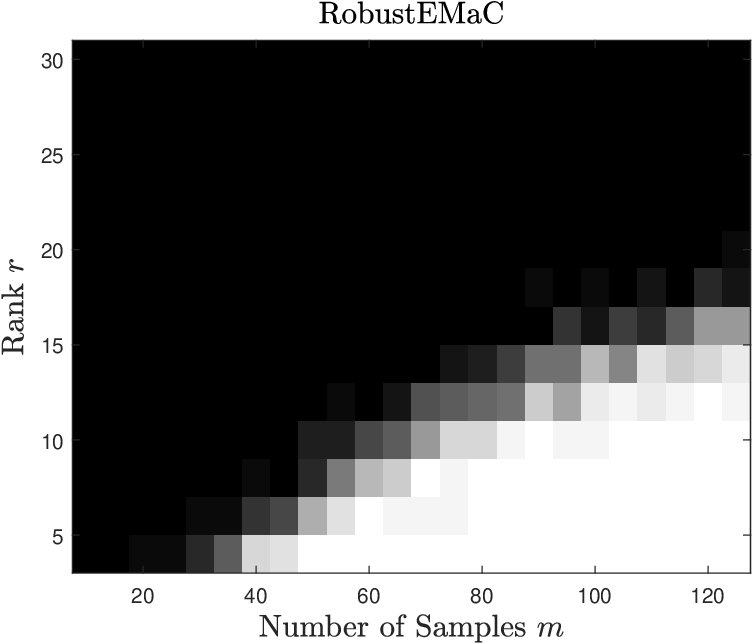}
    \vspace{-0.1in}
    \caption{Phase transition: Number of samples \textit{vs.}~rank. Problem dimension $n=125$, condition number $\kappa=10$, and outlier rate $\alpha=10\%$ for every problem.}
    \label{fig:rank_vs_sample_rate}
\end{figure}

\vspace{0.03in}
\noindent\textbf{Convergence performance.}
In \Cref{fig:speed_tests}, we plot the  relative error v.s. runtime for the tested algorithms with different condition numbers. All problems have dimension $n=2^{16}-1$ and the results are averaged over 5 trials. When the condition number is perfect, i.e., $\kappa=1$, the HSNLD has tied convergence performance with HSGD. When the condition number grows, HSGD's performance drops dramatically while HSNLD takes no impact---this matches our theoretical results. Generally speaking, SAP's convergence performance is not impacted by the condition number either; however, the truncated SVD it applies every iteration may fail to function with vast condition number. RobustEMaC is not tested here because its sublinear convergence is too slow for the comparison.

\begin{figure}[t]
    \centering
    \includegraphics[width = 0.24\linewidth]{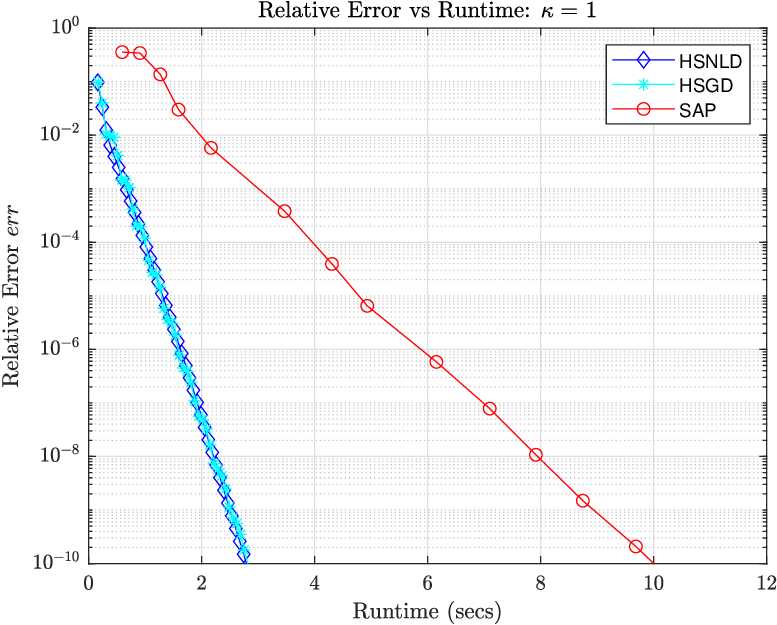}
    \hfill
    \includegraphics[width = 0.24\linewidth]{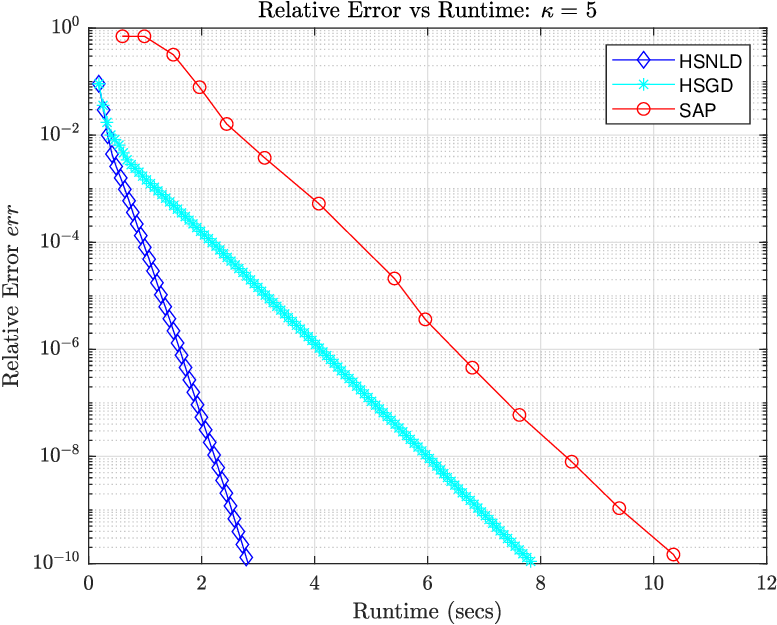}
    \hfill
    \includegraphics[width = 0.24\linewidth]{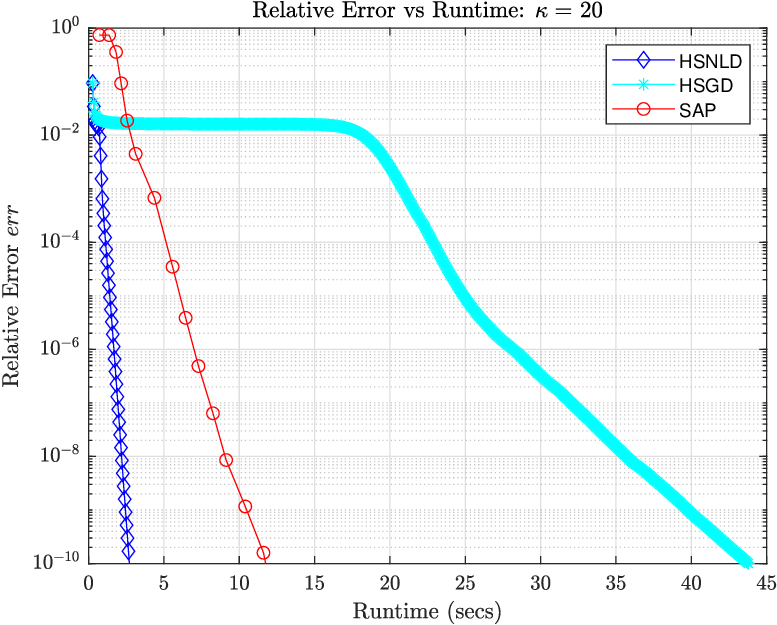}
    \hfill
    \includegraphics[width = 0.24\linewidth]{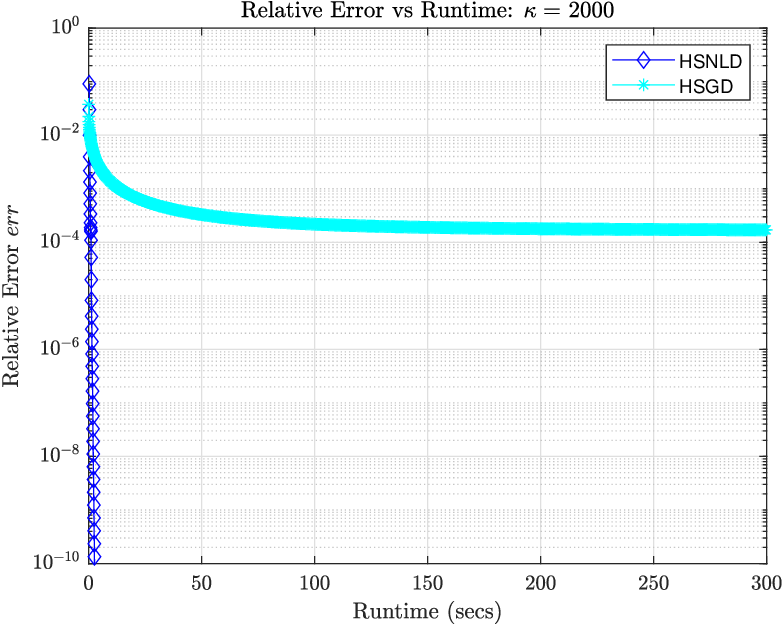}
    \vspace{-0.1in}
    \caption{Convergence performance: Relative error \textit{v.s.}~runtime with condition numbers $\kappa=1,5,20,2000$. Note that the singular values of complex-valued SVD failed to stay real with vast condition number $\kappa = 2000$, thus no convergence result can be reported for SAP.  }
    \label{fig:speed_tests}
\end{figure}

\subsection{Direction of arrival estimation}\label{subsec:DOA} 
Consider $r$ narrowband far-field source signals $\{\tilde{g}_i\}_{i=1}^r\subseteq\mathbb{C}$  impinging on an array of omnidirectional sensors from distinct yet close directions as shown in \Cref{fig:DOA_explain}. Let the DOAs corresponding to these sources be denoted by $\{ \theta_i \}_{i=1}^r$. The noiseless received signal $x_j$ at the $j$-th sensor of a Uniform Linear Array, with a half-wavelength spacing between adjacent sensors, is given by:
\vspace{-0.1in}
\begin{equation*}
x_j = \sum_{i=1}^r \tilde{g}_i e^{-\pi \imath (j-1) \sin \theta_i}, \quad j = 1, 2, \dots, n,
\vspace{-0.05in}
\end{equation*}
where $\imath = \sqrt{-1}$.  Consider the scenario where only a subset $\Omega \subseteq [n]$ of the sensors is available, and the received signals are partially corrupted. The resulting observations are of the form $y_j = x_j+o_j,~j\in \Omega $, where $o_j$ represents the outlier at the $j$-th sensor. The problem of recovering $ \bx= [x_1,x_2,\cdots,x_n]^\top\in\C^n$ from observed data on $\Omega$ (i.e., $\{y_j\}_{j\in \Omega}$) is equivalent to recover $\bz=\mathcal{W}^{-1}\bx$ as defined in \cref{eq:Hankelnotation} and can be formulated as a robust rank-$r$ Hankel matrix completion problem. For the interested reader, please see \cite[Section~6.3]{yang2018sparse} for further details. 

We test HSNLD against the state-of-the-art with DOA data generated from $n = 2^{12}$ sensors and $r = 3$ far-field sources, with directions $\theta = 87^\circ, 87.1^\circ, 87.3^\circ$ as shown in \Cref{fig:DOA_explain}. The underlying square Hankel matrix corresponding to the DOA data has a condition number $\kappa \approx 1083$. In this experiment, only $1.5\%$ sensors are activated (or deployed), and $10\%$ of the collected data are corrupted by complex outliers  whose components are uniformly drawn from the intervals $[ -\sum_{j=1}^n \mathrm{Re}(x_j)/n, \sum_{j=1}^n \mathrm{Re}(x_j)/n ]$ and $[ -\sum_{j=1}^n \mathrm{Im}(x_j)/n, \sum_{j=1}^n \mathrm{Im}(x_j)/n ]$. All test algorithms are terminated at $\|z_k-z\|/\|z\|\leq 10^{-5}$. 
\Cref{fig:doa} plots the signal recovered by HSNLD, compared to the original signal and corrupted partial observations. HSNLD successfully recovers the ground truth. The state-of-the-art algorithms, HSGD and SAP, also successfully recover the DOA signal. Due to space constraints, their plots are omitted. However, the runtime reported in \Cref{tab:doa} shows the significant speed advantage of HSNLD; it achieves nearly $200\times$ speedup to HSGD and $10\times$ speedup to SAP.

\begin{figure}[h]
    \centering
    \includegraphics[width = 0.36\linewidth]{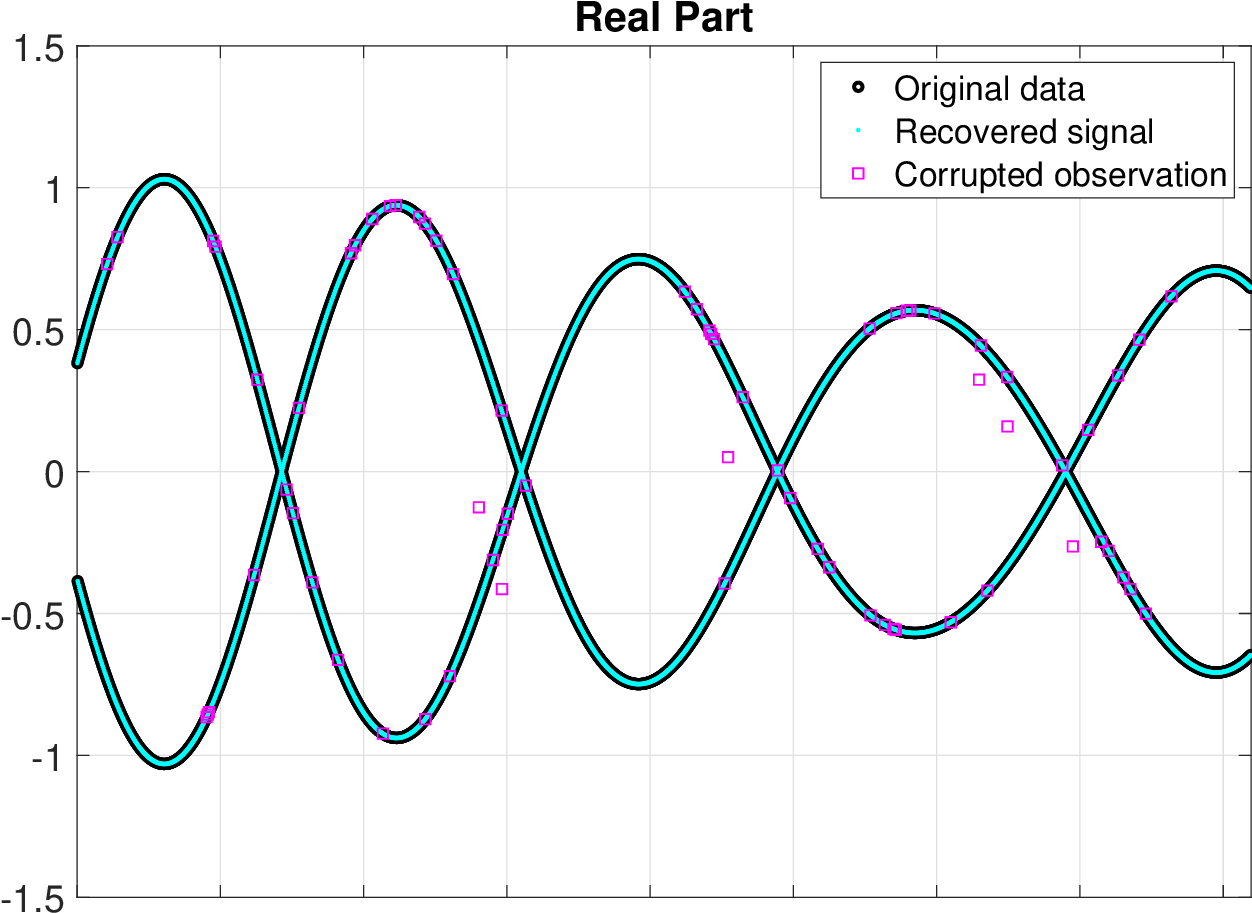}
    \qquad\quad
    \includegraphics[width = 0.36\linewidth]{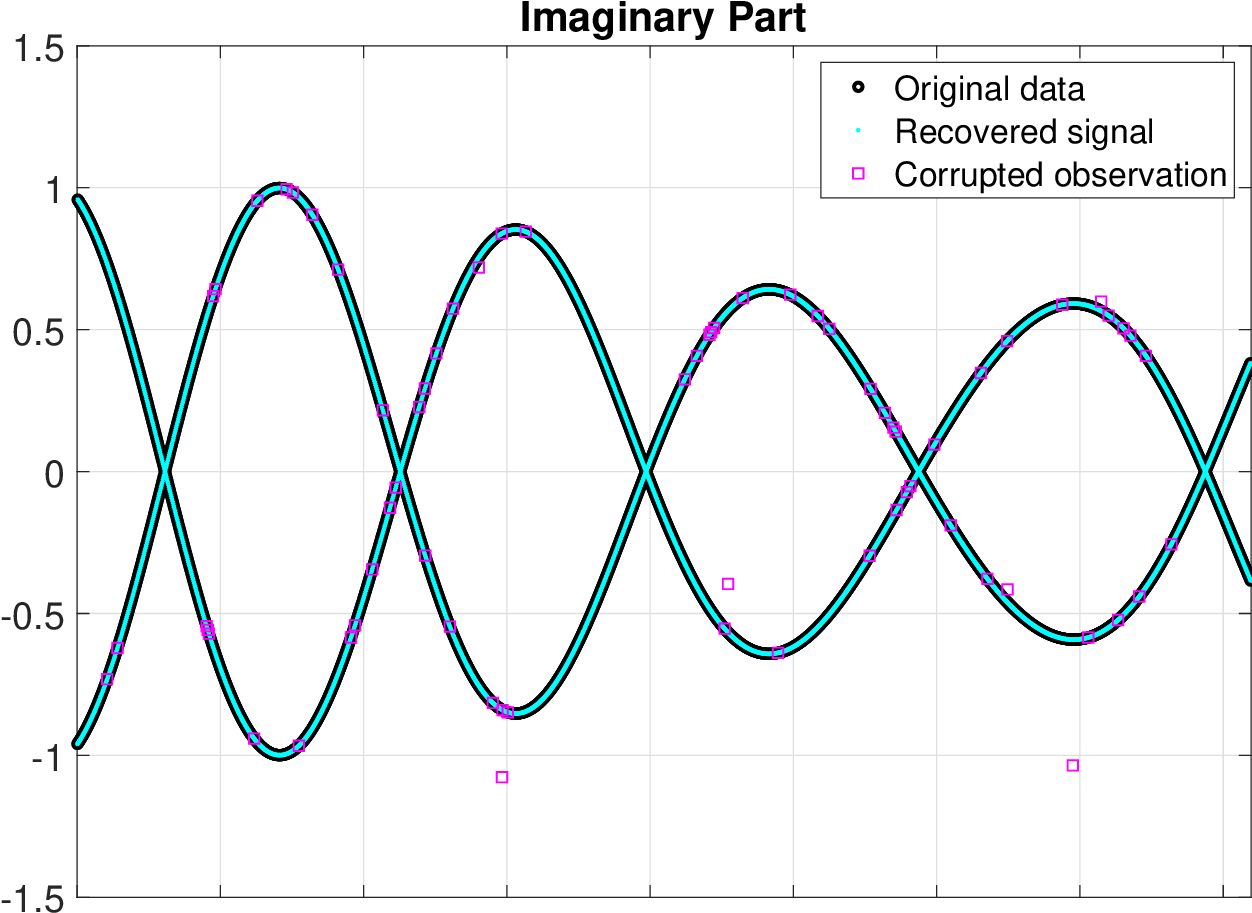}
    \caption{DOA data recover by HSNLD. Only $1.5\%$ of the sensors are activated/deployed and $10\%$ of the observations are corrupted by outliers. Left and right: Real and imaginary parts of the data.}
    \label{fig:doa}
\end{figure}
\begin{table}[h]
    \caption{The number of iterations and runtime for DOA data recovery} \label{tab:doa}
    \centering
    \begin{small}
    \begin{tabular}{c|c c c}
    \toprule
    \textsc{Method}  & HSNLD & HSGD  & SAP  \\
    \midrule
    \textsc{Iteration}  & 72 & 29186 & \textbf{32}\\
    \textsc{Runtime (sec)} & \textbf{0.236} & 38.611  & 2.040\\
    \bottomrule
    \end{tabular}
    \end{small}%
\end{table}

\subsection{Nuclear magnetic resonance} 
We test HSNLD with nuclear magnetic resonance (NMR) signal recovery problem, a classic application of low-rank Hankel recovery. The 1D NMR signal being tested has a length of $32,768$ with an approximate rank of $40$ and condition number $30$. The observation rate is set to be $30\%$ and $10\%$ outliers are added. All algorithms stop at $|\mathrm{err}_{k+1}-\mathrm{err}_{k}|/\mathrm{err}_{k}<10^{-4}$. The recovery results are reported in \Cref{fig:nmr,tab:nmr}. One can see that all three tested algorithms recovered the signal successfully while HSNLD is significantly faster.

\begin{figure}[h]
    \centering
    \includegraphics[width = 0.36\linewidth]{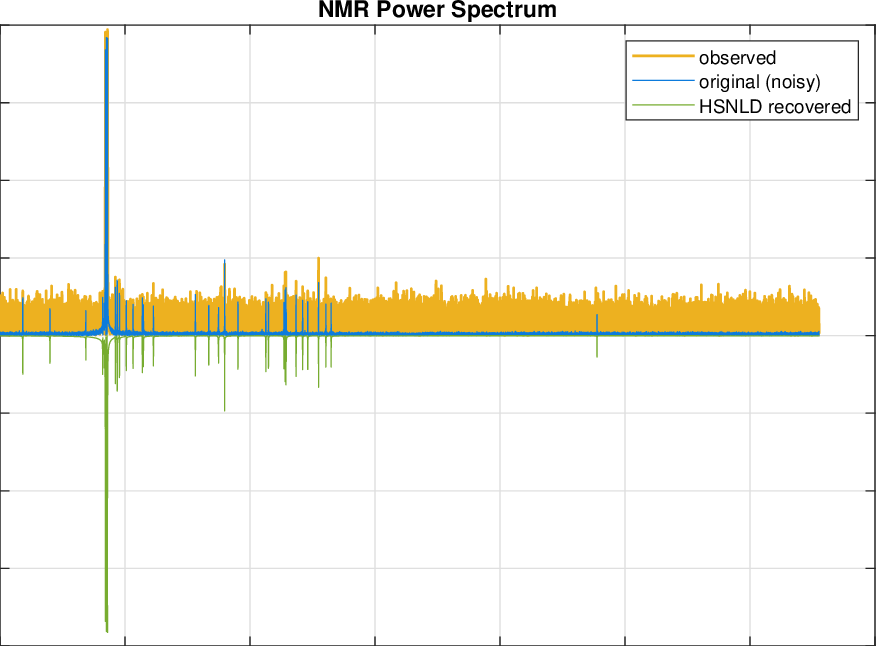}
    \qquad\quad
    \includegraphics[width = 0.36\linewidth]{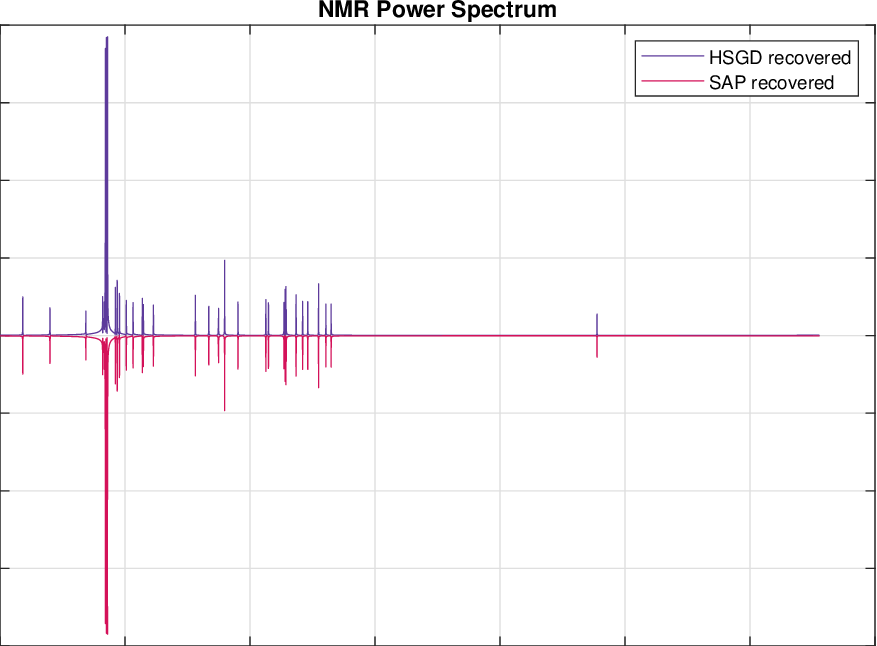}
    \vspace{-0.1in}
    \caption{NMR recovery results in the form of power spectrum.}
    \label{fig:nmr}
\end{figure}
\begin{table}[h]
    \caption{The number of iterations and runtime for NMR recovery.} \label{tab:nmr}
    \centering
    \begin{small}
    \begin{tabular}{c|c c c}
    \toprule
    \textsc{Method}  & HSNLD & HSGD  & SAP  \\
    \midrule
    \textsc{Iteration}  & 43 & 314 & \textbf{40}\\
    \textsc{Runtime (sec)} & \textbf{5.22} & 28.82  & 79.68\\
    \bottomrule
    \end{tabular}
    \end{small}%
\end{table}

\section{Proofs} \label{sec:proofs}
In this section, we provide the analysis of the claimed theoretical results.
All the proofs are under \Cref{amp:Replacement,amp:incoherence,amp:sparsity}. We start with introducing some additional notations used in the analysis.

For ease of presentation, in addition to \Cref{sec:notation}, we denote the following slightly abused notation for the proofs. We let $\Omega$ be any subset of $[n]$ that satisfies \Cref{amp:Replacement} with $m:=\lvert\Omega\rvert$ and $p:=m/n$. We denote the tangent space of rank-$r$ matrix manifold at $\G \bzs$ by
\begin{equation*}
T:=\left\{\bm{X}~|~ \bm{X}=\BU_{\star} \bm{C}^*+\bm{D} \BV_{\star}^{*}\ \text{for any}\ \bm{C}\in\mathbb{C}^{n_1\times r},~ \bm{D}\in\mathbb{C}^{n_2\times r}\right\}.
\end{equation*}
Let $\Omega_a$ denote the index set of the $a$-th antidiagonal of $n_1\times n_2$ matrix, i.e.,
\begin{align}\label{def:Omegaa}
\Omega_a : = \{(i,j):i+j=a+1, 1\le i\le n_1, 1\le j\le n_2\}, ~a=1,2\cdots,n.
\end{align}
We also introduce the following notations for the ease of presentation: Let $\BQ_k$ be the best rotation matrix that aligns $(\BL_{k}, \BR_{k})$ and $(\BLs,\BRs)$, and according to {\cite[Lemma~22]{tong2021accelerating}} we know $\BQ_k$ always exist if $d_k<\sigmas_r$. Denote
\begin{equation}\label{eq:defdelta}
 \BLn := \BL_{k}\BQ_{k},~\BRn := \BR_{k}\BQ_{k}^{-*},~\BDeltaLk := \BLn - \BLs,~\BDeltaRk := \BRn - \BRs.   
\end{equation}

\subsection{Technical lemmas} \label{sec:tech_lemma}
In this subsection, we mainly introduce some supporting lemmas, including bounds estimated from probability inequalities.
\begin{lemma}[{\cite[Lemma~2]{cai2019fast}},~{\cite[Lemma~9]{cai2023structured}}]\label{projectionerror}
It holds with probability at least $1-2n^{-2}$ that
\begin{equation}\label{event:RIP3}
\lV \left(\frac{1}{p}\G \Pi_\Omega -\G\right) \bzs\rV_2\le \tilde{c}_0\sqrt{(\mu c_s r\log n)/m}\lV \G \bzs\rV_2
\end{equation}
 provided $m\ge \mu c_s r\log n$. $\tilde{c}_0$ is a universal constant.
\end{lemma}
\begin{proof}
{The bound of $\lV \left(\frac{1}{p}\G \Pi_\Omega -\G\right) \bzs\rV_2$ is provided in {\cite[Lemma~2]{cai2019fast}} (see {\cite[Eq.~(25)]{cai2019fast}}), and also proved with a slightly different sampling model in \cite[Lemma~9]{cai2023structured}.}
\end{proof}

\begin{lemma}[{\cite[Lemma~10]{cai2023structured}}]\label{RIP1}
 There exists a universal constant $\tilde{c}_1$ such that  
\begin{equation}\label{eq:projectionRIP1}
\lV \frac{1}{p}\P_T\G\Pi_{\Omega}\G^{*}\P_T-\P_T \G\G^{*}\P_T \rV\le \varepsilon_0
\end{equation}
holds with probability at least $1-2n^{-2}$, provided $m\ge \varepsilon_0^{-2}\tilde{c}_1 \mu r\log n$.
\end{lemma}
\begin{proof}
The result is proved in \cite[Lemma~9]{cai2023structured} under a slightly different sampling model and can be extended to our sampling model, which is omitted here for simplicity.
\end{proof}
\begin{lemma}[{\cite[Lemma 5]{cai2018spectral}}]\label{lemma:uv}
For all $\bm{u}\in \mathbb{R}^{n_1}$, $\bm{v}\in \mathbb{R}^{n_2}$, it holds that 
\begin{equation}
\frac{1}{p}\sum_{a_k\in \Omega}\sum_{(i,j)\in \Omega_{a_{k}}}u_i v_j\le \lV\bm{u}\rV_1 \lV\bm{v}\rV_1+\sqrt{\frac{24n\log n}{p}}\lV\bm{u}\rV_2 \lV\bm{v}\rV_2 \label{ineq:uv}
\end{equation}
with probability at least $1-2n^{-2}$, provided $m\ge 3 \log n$.
\end{lemma}

\begin{lemma}\label{projectionerr1}
If $m\ge  \log n$, then under event~\eqref{ineq:uv} it holds that 
\begin{align*}
&~\frac{1}{p}\lv\l\Pi_\Omega\G^*  \left(\BL_{\BA}\BR_{\BA}^*\right),\G^*(\BL_{\BA}\BR_{\BA}^*) \r\rv \cr
\le&~\lV\BL_{\BA}\rV_{\fro}^2\lV\BR_{\BA}\rV_{\fro}^2+\sqrt{\frac{24 n\log n}{p}} \lV\BL_{\BA}\rV_{2,\infty}
\lV\BR_{\BA}\rV_{2,\infty}\lV\BL_{\BA}\rV_\fro\lV\BR_{\BA}\rV_\fro,
\end{align*}
for all $\BL_{\BA}\in\mathbb{C}^{n_1\times r}$ and $\BR_{\BA}\in\mathbb{C}^{n_2\times r}$.
\end{lemma}
\begin{proof}
For $\Omega=\{a_k\}_{k=1}^m$, we have
\begin{align*}
&~\frac{1}{p}\lv\l\Pi_\Omega\G^*  \left(\BL_{\BA}\BR_{\BA}^*\right),\G^*(\BL_{\BA}\BR_{\BA}^*)\r\rv\cr
=&~\frac{1}{p}\sum_{k=1}^m \bigg| \sum_{(i,j)\in\Omega_{a_k}}\frac{1}{\sqrt{\varsigma_{a_k}}}\l \BL_{\BA}\BR_{\BA}^*,\bm{e}_i \bm{e}_j^\top\r \bigg|^2\\
\le&~\frac{1}{p}\sum_{k=1}^m \sum_{(i,j)\in\Omega_{a_k}}\bigg| \l \BL_{\BA}\BR_{\BA}^*,\bm{e}_i \bm{e}_j^\top\r \bigg|^2\\
\le&~\frac{1}{p}\sum_{k=1}^m\sum_{(i,j)\in\Omega_{a_k}}\lV \BL_{\BA}(i,:)\rV_2^2 \lV \BR_{\BA}(j,:)\rV_2^2\\
 \le&~ \lV\BL_{\BA}\rV_{\fro}^2\lV\BR_{\BA}\rV_{\fro}^2+\sqrt{\frac{24n\log n}{p}}\sqrt{\sum_i\lV\BL_{\BA}(i,:)\rV_2^4 }\sqrt{\sum_j\lV\BR_{\BA}(j,:)\rV_2^4}\\
 \le&~ \lV\BL_{\BA}\rV_{\fro}^2\lV\BR_{\BA}\rV_{\fro}^2+\sqrt{\frac{24n\log n}{p}}\lV\BL_{\BA}\rV_{2,\infty}
\lV\BR_{\BA}\rV_{2,\infty}\lV\BL_{\BA}\rV_\fro\lV\BR_{\BA}\rV_\fro.
\end{align*}
where the first and second inequalities follow from the Cauchy-Schwarz inequality, and the third inequality follows from \Cref{lemma:uv}.
\end{proof}

\begin{lemma}\label{RIP2}
 For all matrices $\BA, \BB\in T$,  under event \eqref{eq:projectionRIP1}, it holds that 
 \begin{align*}
 \lv\l \G (\frac{1}{p}\Pi_{\Omega}-\I)\G^*\bm{A},\bm{B}\r\rv \le \varepsilon_0\lV\BA\rV_\fro \lV\BB\rV_\fro,
 \end{align*}
 and
\begin{align*}
\lv  \l\G\Pi_{\Omega}\G^*\BA,\BA\r\rv
\le p(1+\varepsilon_0) \lV\BA\rV_\fro^2.
\end{align*}
\end{lemma}
\begin{proof}
 For any matrices $\BA, \BB\in T$, we have $\P_T \BA=\BA$ and $\P_T \BB=\BB$. Then, under event \eqref{eq:projectionRIP1}, it holds that 
 \begin{align*}
 \lv\l \G (\frac{1}{p}\Pi_{\Omega}-\I)\G^*\BA,\BB\r\rv = \lv\l \P_T\G (\frac{1}{p}\Pi_{\Omega}-\I)\G^*\P_T\BA,\BB\r\rv \le \varepsilon_0\lV\BA\rV_\fro \lV\BB\rV_\fro,
 \end{align*}
Thus, together with the inequality $\lV\G^*\BA\rV_2^2=\l\G\G^*\BA,\BA\r\le \lV\BA\rV_\fro^2$  we have
\begin{equation*}
\begin{aligned}
\lv  \l\G\Pi_{\Omega}\G^*\BA,\BA\r\rv
\le p\varepsilon_0 \lV\BA\rV_\fro^2+p\lV\G^*\BA\rV_2^2
\le p(1+\varepsilon_0) \lV\BA\rV_\fro^2.
\end{aligned}
\end{equation*}
\end{proof}

\begin{lemma}\label{error:Gs}
For any $\bz\in \mathbb{C}^{n}$ and $\bs =\W^{-1}\Gamma_{\gamma\alpha p}(\W\Pi_{\Omega}(\bf-\bz))$ with $\gamma\ge 1$, we have
\begin{equation}\label{s-s}
 \big\|\W\Pi_{\Omega}\bss-\W\bm{s}\big\|_{\infty}\le 2\lV \W\Pi_{\Omega}(\bzs-\bz)\rV_{\infty}.
 \end{equation}
\end{lemma}
\begin{proof}
Recall $\bf = \bzs+ \bss$. By the definition of $\bs$, we have $\W\bs =\Gamma_{\gamma\alpha p}(\W\Pi_{\Omega}(\bzs+ \bss-\bz))$, thus 
\begin{equation*}
\begin{aligned}
\left[\W\Pi_{\Omega}\bss-\W\bs\right]_i =-\left[\W\Pi_{\Omega}\left(\bzs-\bz\right)\right]_i, \quad \text{for}~ i\in \supp\left(\bs\right).
\end{aligned}
\end{equation*}
Given that $\lV \Pi_{\Omega}\bss\rV_0\le\alpha pn =\alpha m$, there can be at most $\alpha m$ elements in $\W\Pi_{\Omega}\left(\bf-\bz\right)$ satisfying  $|[\W\Pi_{\Omega}\left(\bf-\bz\right)]_i|> \lV \W\Pi_{\Omega}\left(\bzs-\bz\right)\rV_{\infty}$. Then, for $\gamma\ge 1$ we have $$|[\W\Pi_{\Omega}\left(\bf-\bz\right)]_i|\le \lV \W\Pi_{\Omega}\left(\bzs-\bz\right)\rV_{\infty}, \quad \text{for}~ i\in \supp(\bss)\setminus\supp\left(\bs\right).$$
 Thus, for $i\in\supp(\bss)\setminus\supp\left(\bs\right)$ we have
\begin{align*}
\left[\W\Pi_{\Omega}\bss-\W\bs\right]_i
&=[\W\Pi_{\Omega}\bss]_i=\left[\W\Pi_{\Omega}(\bf-\bz)-\W\Pi_{\Omega}(\bzs-\bz)\right]_i\cr
&\le |[\W\Pi_{\Omega}(\bf-\bz)]_i|+|[\W\Pi_{\Omega}(\bzs-\bz)]_i|\cr
&\le 2\lV[\W\Pi_{\Omega}(\bzs-\bz)\rV_{\infty}.
\end{align*}
Combing all the pieces we have $\big\|\W\Pi_{\Omega}\bss-\W\bm{s}\big\|_{\infty}\le 2\lV \W\Pi_{\Omega}(\bzs-\bz)\rV_{\infty}$.
\end{proof}
\begin{lemma}\label{PC:Non-expansiveness}(Non-expansiveness of $\P_{\Ca}$) For all $(\widetilde{\BL},\widetilde{\BR})\in \mathcal{E}(\varepsilon_0\sigmas_{r},C)$ with $C\ge  (1+\varepsilon_0) \sqrt{\mu c_s r/n}\sigmas_1$, $\varepsilon_0\in (0,1)$, and $\P_{\Ca}\begin{bmatrix}  \widetilde{\BL}\\ \widetilde{\BR}\end{bmatrix}=\begin{bmatrix}  \BL\\ \BR\end{bmatrix}$, we have
$$\mathrm{dist}\left(\BL,\BR;\BLs,\BRs\right)\le \mathrm{dist}\left(\widetilde{\BL},\widetilde{\BR};\BLs,\BRs\right)$$
and
$$\max\{\|\BL\BR^{*}\|_{2,\infty}, \|\BR\BL^{*}\|_{2,\infty}\}\le   C.$$

\end{lemma}
\begin{proof}
We first show the non-expansiveness of $\P_{\Ca}$. Let $\widetilde{\BQ}\in \mathrm{GL}(r,\C)$ be the best matrix that aligns $\left(\widetilde{\BL},\widetilde{\BR}\right)$ and $\left(\BLs,\BRs\right)$. Notice that for $i\in [n_1]$ we have
\begin{align*}
\big\|\widetilde{\BL}_{i,:}(\widetilde{\BR}^*\widetilde{\BR})^{\frac{1}{2}}\big\|_2 
&\le\big\|\widetilde{\BL}_{i,:}\widetilde{\BQ}\BSigmas^{1/2}\big\|_2 \big\|(\widetilde{\BR}^*\widetilde{\BR})^{\frac{1}{2}}\widetilde{\BQ}^{-*}\BSigmas^{-1/2}\big\|_{2}\cr
&\le \big\|\widetilde{\BL}_{i,:}\widetilde{\BQ}\BSigmas^{1/2}\big\|_2 \big\|(\widetilde{\BR}^*\widetilde{\BR})^{-1/2} \widetilde{\BR}^*\big\|_2 \big\|\widetilde{\BR}\widetilde{\BQ}^{-*}\BSigmas^{-1/2}\big\|_{2}\cr
&\le \big\|\widetilde{\BL}_{i,:}\widetilde{\BQ}\BSigmas^{1/2}\big\|_2  \left(\big\|(\widetilde{\BR}\widetilde{\BQ}^{-*}-\BRs)\BSigmas^{-1/2}\big\|_{2}+\big\|\BRs\BSigmas^{-1/2}\big\|_{2}\right)\cr
&\le (1+\varepsilon_0) \big\|\widetilde{\BL}_{i,:}\widetilde{\BQ}\BSigmas^{1/2}\big\|_2,
\end{align*}
where the third inequality follows from $\big\|(\widetilde{\BR}^*\widetilde{\BR})^{-1/2} \widetilde{\BR}^*\big\|_2=1$, and the last  inequality follows from
\begin{align}\label{bound:deltaLk2}
\lV (\widetilde{\BR}\widetilde{\BQ}^{-*}-\BRs)\BSigmas^{-1/2}  \rV_{2}
&\le\lV (\widetilde{\BR}\widetilde{\BQ}^{-*}-\BRs)\BSigmas^{1/2}   \BSigmas^{-1}\rV_{\fro}\cr
&\le \lV\BSigmas^{-1} \rV_2\mathrm{dist}\left(\widetilde{\BL},\widetilde{\BR};\BLs,\BRs\right)\le \varepsilon_0.
\end{align}
Thus, as long as $C\ge (1+\varepsilon_0) \sqrt{\mu c_s r/n}\sigmas_1$, $\forall i\in [n_1]$ we have
\begin{align}\label{ineq:Liup}
    \frac{C}{\big\|\widetilde{\BL}_{i,:}(\widetilde{\BR}^*\widetilde{\BR})^{\frac{1}{2}}\big\|_2}\ge  \frac{C}{(1+\varepsilon_0) \big\|\widetilde{\BL}_{i,:}\widetilde{\BQ}\BSigmas^{1/2}\big\|_2}\ge \frac{\big\|[\BLs]_{i,:}\BSigmas^{1/2}\big\|_2}{\big\|\widetilde{\BL}_{i,:}\widetilde{\BQ}\BSigmas^{1/2}\big\|_2},
\end{align}
where the last inequality follows from $\BLs\BSigmas^{1/2}=\BU_{\star}\BSigmas$ and 
\begin{align*}
\lV[\BLs]_{i,:}\BSigmas^{1/2}\rV_{2}\le\lV\BU_{\star}\rV_{2,\infty} \lV\BSigmas\rV_{2}\le \sqrt{\mu c_s r/n}\sigmas_1,\quad \forall i\in [n_1].
\end{align*}
For any vectors $\bm{u},\bm{v}\in\mathbb{R}^n$, it holds  that $\lV\min(1,\lambda)\bm{u}- \bm{v}\rV_2 \le \lV\bm{u}- \bm{v}\rV_2$ provided $\lambda\ge \frac{\lV\bm{v}\rV_2}{\lV\bm{u}\rV_2}$ (see \cite[Claim 5]{tong2021accelerating}). Recalling that $\BL_{i,:}=\min\left(1,\frac{C}{\big\|\widetilde{\BL}_{i,:}(\widetilde{\BR}^*\widetilde{\BR})^{\frac{1}{2}}\big\|_2}\right)\widetilde{\BL}_{i,:}$, thus by \Cref{ineq:Liup} we have
\begin{align*}
 \lV\left(\BL_{i,:}\widetilde{\BQ} -[\BLs]_{i,:}\right)\BSigmas^{1/2} \rV_2^2\le \lV\left(\widetilde{\BL}_{i,:}\widetilde{\BQ} -[\BLs]_{i,:}\right)\BSigmas^{1/2} \rV_2^2
\end{align*}
and similarly  $\lV\left(\BR_{i,:}\widetilde{\BQ}^{-*} -[\BRs]_{i,:}\right)\BSigmas^{1/2} \rV_2^2\le \lV\left(\widetilde{\BR}_{i,:}\widetilde{\BQ}^{-*} -[\BRs]_{i,:}\right)\BSigmas^{1/2} \rV_2^2$.
Then, by definition of the metric we have the local non-expansiveness of $\P_{\Ca}$, i.e.,
\begin{align*}
 \mathrm{dist}^2\left(\BL,\BR;\BLs,\BRs\right) 
 &\le \big\| \left(\BL\widetilde{\BQ}-\BLs\right)\BSigmas^{1/2}  \big\|_\fro^2+\big\| \left(\BR\widetilde{\BQ}^{-*}-\BRs\right)\BSigmas^{1/2} \big\|_\fro^2\cr
 &\le \mathrm{dist}^2\left(\widetilde{\BL},\widetilde{\BR};\BLs,\BRs\right).
 \end{align*}
Noticing that $\|\BL_{i,:}(\widetilde{\BR}^{*}\widetilde{\BR})^{1/2}\|_{2}=\|\min(1,\frac{C}{\|\widetilde{\BL}_{i,:}(\widetilde{\BR}^*\widetilde{\BR})^{\frac{1}{2}}\|_2})\widetilde{\BL}_{i,:}(\widetilde{\BR}^{*}\widetilde{\BR})^{1/2}\|_{2}\le C$, $\forall i\in [n_1]$,
thus
 $\|\BL\BR^{*}\|_{2,\infty}\le \|\BL(\widetilde{\BR}^{*}\widetilde{\BR})^{1/2}\|_{2,\infty}\| \BR(\widetilde{\BR}^{*}\widetilde{\BR})^{-1/2}\|_2\le C$. Similarly, we have $\|\BR\BL^{*}\|_{2,\infty}\le C$.
\end{proof}

\subsection{Proof of \texorpdfstring{\Cref{prop:dist}}{Proposition~\protect\ref{prop:dist}}}\label{subsec:proofdist}
\begin{proof}
For $d_k \le \varepsilon_0 \sigmas_r$, following \Cref{bound:deltaLk2}, we have
$\lV \BDeltaLk\BSigmas^{-1/2}  \rV_{2}\le \varepsilon_0,$ thus  $$\lV\BDeltaLk\BDeltaRk^* \rV_{\fro} = \lV\BDeltaLk\BSigmas^{-1/2} (\BDeltaRk\BSigmas^{1/2})^*  \rV_{\fro}\le \varepsilon_0 \lV\BDeltaRk\BSigmas^{1/2}  \rV_{\fro}.$$  
Similarly, we have $\lV \BDeltaRk\BSigmas^{-1/2} \rV_{2}\le \varepsilon_0$ and $\lV\BDeltaLk\BDeltaRk^* \rV_{\fro} \le \varepsilon_0 \lV\BDeltaLk\BSigmas^{1/2}  \rV_{\fro}$. Therefore, we have
\begin{align}\label{bound:deltaLk3}
\lV\BL_{k}\BR_{k}^*-\G\bzs\rV_{\fro}
=&~\lV\BDeltaLk\BRs^*+\BLs\BDeltaRk^*+\BDeltaLk\BDeltaRk^*\rV_{\fro}\cr
\le&~\lV\BDeltaLk\BSigmas^{1/2}\BV_{\star}^* \rV_{\fro}+\lV\BU_{\star}\BSigmas^{1/2}\BDeltaRk^* \rV_{\fro}+\lV\BDeltaLk\BDeltaRk^* \rV_{\fro} \cr
\le &~ (1+\frac{\varepsilon_0}{2}) \left( \lV \BDeltaLk\BSigmas^{1/2}\rV_\fro+\lV \BDeltaRk\BSigmas^{1/2}\rV_\fro\right).
\end{align}
Thus, by noticing $\left( \lV \BDeltaLk\BSigmas^{1/2}\rV_\fro+\lV \BDeltaRk\BSigmas^{1/2}\rV_\fro\right)^2\le 2 d_k^2$, we complete the proof.
\end{proof}

\subsection{Proof of \texorpdfstring{\Cref{thm:Initialization}}{Theorem~\protect\ref{thm:Initialization}} (guaranteed initialization)}\label{subsec:proofini}

\begin{proof}
This proof is under event \eqref{event:RIP3} and we bound $d_0$ in the following. Following the notation of \Cref{algo:HankelscaledGD}, we denote
\begin{equation}\label{def-M}
 \bm{M}_{0}:=\frac{1}{p}\G\left(\Pi_{\Omega}\bf-\bm{s}_{0}\right),\quad
 \text{and}~\bm{M}_{0}^{(r)}~\text{be the top-$r$ SVD of}~ \bm{M}_{0}.
\end{equation}
For some universal constant $\tilde{c}_0$, provided $m\ge 64\tilde{c}_0^2\varepsilon_0^{-2}\kappa^2 c_s\mu r^2\log n$, we have
\begin{align}\label{term-tri}
\lV \bm{M}_{0}-\G \bzs\rV_2
&\le \lV \bm{M}_{0}-\frac{1}{p}\G \Pi_\Omega \bzs\rV_2 +\lV \frac{1}{p}\G \Pi_\Omega \bzs-\G \bzs\rV_2\cr
&= \lV \frac{1}{p}\G\left(\Pi_{\Omega}\bss-\bm{s}_{0}\right)\rV_2 +\lV \left(\frac{1}{p}\G \Pi_\Omega -\G \right)\bzs\rV_2\cr
&\le 2\alpha n\big\| \G\big(\Pi_{\Omega}\bss-\bm{s}_{0}\big)\big\|_{\infty} +\tilde{c}_0\sqrt{(\mu c_s r\log n)/m}\sigmas_{1}\cr
&\le  4\alpha n\big\| \G \bzs\big\|_{\infty}+0.125\varepsilon_0\sigmas_{r}r^{-\frac{1}{2}}\cr
&\le  4\alpha c_s\mu r \sigmas_{1}+0.125\varepsilon_0\sigmas_{r} r^{-\frac{1}{2}},
\end{align}
where the second step follows from
$\| \G\left(\Pi_{\Omega}\bf-\bm{s}_{0}\right)-\G \Pi_\Omega \bzs\|_2
=\| \G\left(\Pi_{\Omega}\bss-\bm{s}_{0}\right)\|_2$. The third line follows from {\cite[Lemma~6]{cai2021asap}} ( as $\G=\H\W$ and $\W(\Pi_{\Omega}\bss-\bm{s}_{0})$ is at most $2\alpha m$-sparse) and \Cref{projectionerror}. The fourth line follows from
\eqref{s-s} with $\bz=\bm{0}$, i.e., $\lV\G \left(\Pi_{\Omega}\bss-\bm{s}_{0}\right)\rV_{\infty}=\lV\W \left(\Pi_{\Omega}\bss-\bm{s}_{0}\right)\rV_{\infty}\le 2\lV \W\Pi_{\Omega}\bzs\rV_{\infty}= 2\lV \G \Pi_{\Omega}\bzs\rV_{\infty}$, because $\lV\G\bm{z}\rV_{\infty}$ and $\lV\W\bm{z}\rV_{\infty}$ with $\W\bm{z}=\bm{x}$ both equal the largest entry-wise magnitude in $\bm{x}$, see \eqref{eq:Hankelnotation}. And the last inequality follows from
$$\big\| \G \bzs\big\|_{\infty} \le \big\| \BU_\star \big\|_{2,\infty} \big\|\G \bzs\big\|_{2} \big\|\BV_\star\big\|_{2,\infty}\le \mu \frac{c_sr}{n}\lV \G \bzs\rV_{2}.$$
Therefore, we have
\begin{align}\label{ineq:M0rGz}
\lV \bm{M}_{0}^{(r)}-\G \bzs\rV_2 &\le \lV \bm{M}_{0}^{(r)}-\bm{M}_{0}\rV_2 +\lV \bm{M}_{0}-\G \bzs\rV_2 \cr
 &\le 2\lV \bm{M}_{0}-\G \bzs\rV_2
 ~\le 8\alpha c_s\mu r\kappa \sigmas_{r}+0.25\varepsilon_0\sigmas_rr^{-\frac{1}{2}},
 \end{align}
where the second inequality follows from the definition of $\bm{M}_{0}^{(r)}$ and Eckart-Young-Mirsky theorem.
Then, by \cite[Lemma~24]{tong2021accelerating} we have
\begin{align*}
&~\mathrm{dist}\left(\BU_{0}\BSigma_{0}^{1/2},\BV_{0}\BSigma_{0}^{1/2};\BLs,\BRs\right)
\le (\sqrt{2}+1)^{1/2} \lV \bm{M}_{0}^{(r)}-\G \bzs\rV_\fro\cr
\le&~ \sqrt{2(\sqrt{2}+1)r} \lV \bm{M}_{0}^{(r)}-\G \bzs\rV_2\le  18\alpha c_s\kappa \mu r\sqrt{r} \sigmas_{r}+0.55\varepsilon_0\sigmas_{r}.
\end{align*}
Then, for any given $\varepsilon_0\in (0,1)$, we have $\mathrm{dist}\left(\BU_{0}\BSigma_{0}^{1/2},\BV_{0}\BSigma_{0}^{1/2};\BLs,\BRs\right)\le \varepsilon_0\sigmas_{r}<\sigmas_{r}$ provided $0\le \alpha\le \frac{\varepsilon_0}{50 c_s\kappa \mu r\sqrt{r}}$. Thus, for $c\ge 1+\varepsilon_0$, by \Cref{PC:Non-expansiveness} we have 
\begin{equation*}
d_0\le \mathrm{dist}\left(\BU_{0}\BSigma_{0}^{1/2},\BV_{0}\BSigma_{0}^{1/2};\BLs,\BRs\right) \le \varepsilon_0\sigmas_{r},
\end{equation*}
and $
\max\{\lV\BL_0\BR_0^{*}\rV_{2,\infty},\lV\BR_0\BL_0^{*}\rV_{2,\infty}\}\le   c \sqrt{\frac{c_s \mu r}{n}}\sigmas_1
$. This completes the proof.
\end{proof}

\subsection{Proof of \texorpdfstring{\Cref{thm:convergence_full,thm:convergence}}{Theorem~\protect\ref{thm:convergence}} (local convergence)}\label{subsec:prooflocalcon}
In this subsection, based on the probabilistic events in~\Cref{sec:tech_lemma}, we prove several inequalities that are crucial for the proof of our main theorems. The proof follows a similar route established in \cite{tong2021accelerating}, but the techniques and details are quite involved due to the extra Hankel structure, as well as the simultaneous presence of outliers and missing data.  

In the following, we let $\varepsilon_0\in (0,1)$ be some fixed constant, and show that for any given $(\BL_{k},\BR_{k})\in \mathcal{E}(\varepsilon_0 \sigmas_r,c \sqrt{\frac{c_s \mu r}{n}}\sigmas_1)$ we have the contraction $d_{k+1}\le \tau d_k$ for some $\tau\in (0,1)$ provided $\varepsilon_0$ sufficiently small and $c\ge 1+\varepsilon_0$. We begin with some auxiliary lemmas.
\begin{lemma}\label{lemma:boundnorms}
For any $(\BL_{k},\BR_{k})\in \mathcal{E}(\varepsilon_0 \sigmas_r,c \sqrt{\frac{c_s \mu r}{n}}\sigmas_1)$, we have
\begin{subequations}
\begin{align}
\max\{\lV \BDeltaLk\BSigmas^{-1/2}  \rV_{2},\lV \BDeltaRk\BSigmas^{-1/2} \rV_{2}\}&\le \varepsilon_0, \label{boundnorm:a}\\
\max\left\{\lV\BLn\BSigmas^{1/2}\rV_{2,\infty}, \lV\BRn\BSigmas^{1/2}\rV_{2,\infty}\right\} &\le \frac{c}{1-\varepsilon_0}\sqrt{\frac{c_s\mu r}{n}}\sigmas_1,\label{boundnorm:b}\\
\max\left\{\lV\BLn\BSigmas^{-1/2}\rV_{2,\infty}, \lV\BRn\BSigmas^{-1/2}\rV_{2,\infty}\right\} &\le \frac{c\kappa}{1-\varepsilon_0}\sqrt{\frac{c_s\mu r}{n}},\label{boundnorm:c}\\
\lV\BRn\left(\BRn^{*}\BRn\right)^{-1} \BSigmas^{1/2}\rV_2 &\le \frac{1}{1-\varepsilon_0}\label{boundnorm:d}.
\end{align}
\end{subequations}
\end{lemma}
\begin{proof}
The inequality \eqref{boundnorm:a} follows from \eqref{bound:deltaLk2}. To show \eqref{boundnorm:b}, noticing that $$\sigma_r(\BLn\BSigmas^{-1/2})\ge \sigma_r(\BLs\BSigmas^{-1/2})-\sigma_1(\BDeltaLk\BSigmas^{-1/2})\ge 1-\varepsilon_0,$$
then by $\lV\BL_{k}\BR_{k}^{*}\rV_{2,\infty}\ge \sigma_r(\BRn\BSigmas^{-1/2})\lV\BLn\BSigmas^{1/2}\rV_{2,\infty}$ we have $\lV\BLn\BSigmas^{1/2}\rV_{2,\infty}\le \frac{\lV\BL_{k}\BR_{k}^{*}\rV_{2,\infty}}{1-\varepsilon_0}$. Similarly, we
 have $\lV\BRn\BSigmas^{1/2}\rV_{2,\infty}\le \frac{\lV\BL_{k}\BR_{k}^{*}\rV_{2,\infty}}{1-\varepsilon_0}$ and conclude \eqref{boundnorm:b}. By using $\lV\BA\BB\rV_{2,\infty}\le \lV\BA\rV_{2,\infty}\lV\BB\rV_2$, \eqref{boundnorm:c} follows directly from \eqref{boundnorm:b}. \eqref{boundnorm:d} follows from {\cite[Lemma~25]{tong2021accelerating}} and \eqref{boundnorm:a}.
\end{proof}


\begin{lemma}\label{projectioner:Gdelta}
For all matrices $(\BL_{k},\BR_{k})\in \mathcal{E}(\varepsilon_0 \sigmas_r,c \sqrt{\frac{c_s \mu r}{n}}\sigmas_1)$, it holds that 
\begin{align*}
\lv\l\Pi_{\Omega}\G^*\left(\BL_{k}\BR_{k}^*- \G\bzs\right),\G^*\left(\BL_{k}\BR_{k}^*- \G\bzs\right)\r\rv^{\frac{1}{2}}\cr
\le  \sqrt{p}(1+3\sqrt{\varepsilon_0})\left(\lV\BDeltaLk\BSigmas^{1/2}\rV_{\fro}+\lV\BDeltaRk\BSigmas^{1/2}\rV_\fro\right)
\end{align*}
provided $m\ge 6C_{\varepsilon_0,\kappa}^2\varepsilon_0^{-2}c_s^2 \mu^2 r^2\log n $, where $C_{\varepsilon_0,\kappa}=\left(1+\frac{c}{1-\varepsilon_0}\right)\left(1+\frac{c\kappa}{1-\varepsilon_0}\right)\kappa$.
\end{lemma}
\begin{proof}
The decomposition $\BL_{k}\BR_{k}^*- \G\bzs= \BDeltaLk\BRs^*+\BLs\BDeltaRk^*+\BDeltaLk\BDeltaRk^*$ gives
\begin{small}
 \begin{align*}
&~\lv\l\Pi_{\Omega}\G^*\left(\BL_{k}\BR_{k}^*- \G\bzs\right),\G^*\left(\BL_{k}\BR_{k}^*- \G\bzs\right)\r\rv^{\frac{1}{2}} \cr
\le&~\lv\l\Pi_{\Omega}\G^*\left(\BDeltaLk\BRs^*+\BLs\BDeltaRk^*\right),\G^*\left(\BDeltaLk\BRs^*+\BLs\BDeltaRk^*\right)\r\rv^{\frac{1}{2}}\cr
&~+\lv\l\Pi_{\Omega}\G^*\left(\BDeltaLk\BDeltaRk^*\right),\G^*\left(\BDeltaLk\BDeltaRk^*\right)\r\rv^{\frac{1}{2}}\cr
\le &~ \sqrt{p\left(1+\varepsilon_0\right)}\lV\left(\BDeltaLk\BSigmas^{1/2}\BV_{\star}^*+\BU_{\star}\BSigmas^{1/2}\BDeltaRk^*\right)\rV_\fro
+\lv\l\Pi_{\Omega}\G^*\left(\BDeltaLk\BDeltaRk^*\right),\G^*\left(\BDeltaLk\BDeltaRk^*\right)\r\rv^{\frac{1}{2}}\cr
\le &~\sqrt{p\left(1+\varepsilon_0\right)}\left(\lV\BDeltaLk\BSigmas^{1/2}\rV_{\fro}+\lV\BDeltaRk\BSigmas^{1/2}\rV_\fro\right)+\lv\l\Pi_{\Omega}\G^*\left(\BDeltaLk\BDeltaRk^*\right),\G^*\left(\BDeltaLk\BDeltaRk^*\right)\r\rv^{\frac{1}{2}},
\end{align*}
\end{small}%
where the first step follows from $\lv\l\Pi_{\Omega}(\bm{u}+\bm{v}),\bm{u}+\bm{v}\r\rv^{\frac{1}{2}}\le\lv\l\Pi_{\Omega}\bm{u},\bm{u}\r\rv^{\frac{1}{2}}+\lv\l\Pi_{\Omega}\bm{v},\bm{v}\r\rv^{\frac{1}{2}}$ for $\bm{u}, \bm{v}\in\mathbb{C}^m$, the second inequality follows from \Cref{RIP2}, and in the last inequality we have used $\lV\BDeltaLk\BSigmas^{1/2}\BV_{\star}^*\rV_{\fro}=\lV\BDeltaLk\BSigmas^{1/2}\rV_{\fro}$. For the last term, provided $m\ge 6\varepsilon_0^{-2}C_{\varepsilon_0,\kappa}^2 c_s^2 \mu^2 r^2\log n $, we have
 \begin{align*}
&~\frac{1}{p}\lv\l\Pi_{\Omega}\G^*\left(\BDeltaLk\BDeltaRk^*\right),\G^*\left(\BDeltaLk\BDeltaRk^*\right)\r\rv \cr
\le&~\lV\BDeltaLk\BSigmas^{-1/2}\rV_{2}^{2}\lV\BDeltaRk\BSigmas^{1/2}\rV_{\fro}^2\cr
&~+\sqrt{\frac{24 n\log n}{p}} \lV\BDeltaLk\BSigmas^{1/2}\rV_{2,\infty}
\lV\BDeltaRk\BSigmas^{-1/2}\rV_{2,\infty}\lV\BDeltaLk\BSigmas^{1/2}\rV_\fro\lV\BDeltaRk\BSigmas^{-1/2}\rV_\fro\cr
\le &~ \lV\BDeltaLk\BSigmas^{-1/2}\rV_{2}^{2}\lV\BDeltaRk\BSigmas^{1/2}\rV_{\fro}^2
+\varepsilon_0 \left(\lV\BDeltaLk\BSigmas^{1/2}\rV_{\fro}^2+
\lV\BDeltaRk\BSigmas^{1/2}\rV_\fro^2\right)
\end{align*}
where the first inequality follows from \Cref{projectionerr1}, the second inequality follows from the inequalities $\lV\BDeltaRk\BSigmas^{-1/2}\rV_\fro = \lV\BDeltaRk\BSigmas^{1/2}\BSigmas^{-1}\rV_\fro\le \frac{1}{\sigmas_r}\lV\BDeltaRk\BSigmas^{1/2}\rV_\fro$, and  (by \eqref{boundnorm:c})
\begin{equation} \label{ineq:DeltaLl2inf}
\begin{split}
\lV\BDeltaLk\BSigmas^{-\frac{1}{2}}\rV_{2,\infty} 
&\le\lV\BLn\BSigmas^{-\frac{1}{2}}\rV_{2,\infty}+\lV\BLs\BSigmas^{-\frac{1}{2}}\rV_{2,\infty}\le \left(1+\frac{c\kappa}{1-\varepsilon_0}\right)\sqrt{\frac{c_s\mu r}{n}},\cr
\lV\BDeltaLk\BSigmas^{\frac{1}{2}}\rV_{2,\infty}  
&\le\lV\BLn\BSigmas^{\frac{1}{2}}\rV_{2,\infty}+\lV\BLs\BSigmas^{\frac{1}{2}}\rV_{2,\infty}\le \left(1+\frac{c}{1-\varepsilon_0}\right)\sqrt{\frac{c_s\mu r}{n}}\sigmas_1.
\end{split}
\end{equation}
Therefore,
\begin{align*}
&~\lv\l\Pi_{\Omega}\G^*\left(\BDeltaLk\BDeltaRk^*\right),\G^*\left(\BDeltaLk\BDeltaRk^*\right)\r\rv^{\frac{1}{2}} \cr
\le &~ \sqrt{p}\lV\BDeltaLk\BSigmas^{-1/2}\rV_{2}\lV\BDeltaRk\BSigmas^{1/2}\rV_{\fro}
+\sqrt{\varepsilon_0 p} \left(\lV\BDeltaLk\BSigmas^{1/2}\rV_{\fro}+
\lV\BDeltaRk\BSigmas^{1/2}\rV_\fro\right).
\end{align*}
Similarly, we have
\begin{align*}
&~\lv\l\Pi_{\Omega}\G^*\left(\BDeltaLk\BDeltaRk^*\right),\G^*\left(\BDeltaLk\BDeltaRk^*\right)\r\rv^{\frac{1}{2}} \cr
\le &~ \sqrt{p}\lV\BDeltaLk\BSigmas^{1/2}\rV_{\fro}\lV\BDeltaRk\BSigmas^{-1/2}\rV_{2}
+\sqrt{\varepsilon_0 p} \left(\lV\BDeltaLk\BSigmas^{1/2}\rV_{\fro}+
\lV\BDeltaRk\BSigmas^{1/2}\rV_\fro\right).
\end{align*}
By \eqref{boundnorm:a} and averaging the above two bounds, we conclude
\begin{align*}
\lv\l\Pi_{\Omega}\G^*\left(\BDeltaLk\BDeltaRk^*\right),\G^*\left(\BDeltaLk\BDeltaRk^*\right)\r\rv^{\frac{1}{2}}
\le 2\sqrt{p\varepsilon_0}\left(\lV\BDeltaLk\BSigmas^{1/2}\rV_{\fro}+\lV\BDeltaRk\BSigmas^{1/2}\rV_\fro\right).
\end{align*}
Putting all the bound of terms together, we conclude the lemma.
\end{proof}

\begin{lemma}\label{PZUZV-M}
 For any indices set $\widehat{\Omega}\subseteq [n]$ we have
 \begin{align}\label{ineq:PiGDelta}
&~\lv\l\Pi_{\widehat{\Omega}}\G^{*}\left(\BL_{k} \BR_{k}^*-\G\bzs\right),\G^{*}\left(\BL_{k} \BR_{k}^*-\G\bzs\right)\r \rv^{\frac{1}{2}}\cr
\le&~ \left(1+\frac{c\kappa}{1-\varepsilon_0}\right)\sqrt{c_s\mu r |\widehat{\Omega}|/(2n)}\left( \lV \BDeltaLk\BSigmas^{1/2}\rV_\fro+\lV \BDeltaRk\BSigmas^{1/2}\rV_\fro\right),
\end{align}
and for all $\BL\in \mathbb{C}^{n_1\times r},~\BR\in \mathbb{C}^{n_2\times r}$ it holds that
$$\lv\l\Pi_{\widehat{\Omega}}\G^{*}\left(\bm{L} \bm{R}^*\right),\G^{*}\left(\bm{L} \bm{R}^*\right)\r \rv ^{\frac{1}{2}}\le \sqrt{|\widehat{\Omega}|} \min\{\lV\BL\rV_{2,\infty}\lV\BR\rV_\fro,\lV\BR\rV_{2,\infty}\lV\BL\rV_{\fro}\}.  $$
\end{lemma}
\begin{proof}
By using the decomposition $\BL_{k}\BR_{k}^*- \G\bzs= (\BLn-\BLs)\BRn^*+\BLs(\BRn-\BRs)= (\BLn-\BLs)\BRs^*+\BLn(\BRn-\BRs)$, we have
\begin{align*}
&~2\lv\left[\BLn\BRn^*- \G\bzs\right]_{i,j}\rv\cr
=&~\lv\left[(\BLn-\BLs)\BRn^*+\BLs(\BRn-\BRs)\right]_{i,j}+\left[(\BLn-\BLs)\BRs^*+\BLn(\BRn-\BRs)\right]_{i,j}\rv \\
\le&~ \lV \BRn\BSigmas^{-1/2}\rV_{2,\infty} \lV [\BDeltaLk\BSigmas^{1/2}]\left(i,:\right)\rV_2+\lV \BLs\BSigmas^{-1/2}\rV_{2,\infty} \lV [\BDeltaRk\BSigmas^{1/2}]\left(j,:\right)\rV_2\cr
&~+\lV \BRs\BSigmas^{-1/2}\rV_{2,\infty} \lV [\BDeltaLk\BSigmas^{1/2}]\left(i,:\right)\rV_2+\lV \BLn\BSigmas^{-1/2}\rV_{2,\infty} \lV [\BDeltaRk\BSigmas^{1/2}]\left(j,:\right)\rV_2\cr
\le &~\theta\left(\lV [\BDeltaLk\BSigmas^{1/2}]\left(i,:\right)\rV_2+\lV [\BDeltaRk\BSigmas^{1/2}]\left(j,:\right)\rV_2\right).
\end{align*}
where 
\begin{align*}
 \theta&= \max\{\lV \BRn\BSigmas^{-1/2}\rV_{2,\infty}+\lV \BRs\BSigmas^{-1/2}\rV_{2,\infty}, \lV \BLs\BSigmas^{-1/2}\rV_{2,\infty}+ \lV \BLn\BSigmas^{-1/2}\rV_{2,\infty}\}\cr
 &\le \left(1+\frac{c\kappa}{1-\varepsilon_0}\right)\sqrt{c_s\mu r/n}   
\end{align*}
due to \Cref{lemma:boundnorms}. Let $\widehat{\Omega}=\{\hat{a}_t\}_{1\le t\le|\widehat{\Omega}|}$,  denote $\widehat{\Phi}_t:=\{\left(i,j\right):\left(i+j-1\right)= \hat{a}_t\}$ and $\widehat{\Phi}=\cup_{1\le t\le|\widehat{\Omega}|}\widehat{\Phi}_t$. Then, by definition in \eqref{eq:defdelta}, we have
\begin{align*}
&~\lv\l\Pi_{\widehat{\Omega}}\G^{*}\left(\BL_{k} \BR_{k}^*-\G\bzs\right),\G^{*}\left(\BL_{k} \BR_{k}^*-\G\bzs\right)\r \rv
=\sum_{\hat{a}_t\in \widehat{\Omega}}\frac{1}{\varsigma_{\hat{a}_t}}\lv \sum_{\left(i,j\right)\in \widehat{\Phi}_t}[\BLn\BRn^*- \G\bzs]_{i,j}\rv^2\\
 \le& \sum_{\hat{a}_t\in \widehat{\Omega}}\sum_{\left(i,j\right)\in \widehat{\Phi}_t}\lv [\BLn\BRn^*- \G\bzs]^2_{i,j}\rv
 = \sum_{\left(i,j\right)\in \widehat{\Phi}}\lv [\BLn\BRn^*- \G\bzs]^2_{i,j}\rv\\
  \le&~\sum_{\left(i,j\right)\in \widehat{\Phi}}\frac{\theta^2}{4}\left(\lV [\BDeltaLk\BSigmas^{1/2}]\left(i,:\right)\rV_2+\lV [\BDeltaRk\BSigmas^{1/2}]\left(j,:\right)\rV_2\right)^2\cr
  \le&~\sum_{\left(i,j\right)\in \widehat{\Phi}}\frac{\theta^2}{2}\left(\lV [\BDeltaLk\BSigmas^{1/2}]\left(i,:\right)\rV_2^2+\lV [\BDeltaRk\BSigmas^{1/2}]\left(j,:\right)\rV_2^2\right)\cr
   = &~\frac{\theta^2}{2}\Bigg(\sum_i \sum_{1\leq t\leq|\hat{\Omega}|}\chi_{[1,\infty)}(\hat{a}_t+1-i)
 \lV [\BDeltaLk\BSigmas^{1/2}]\left(i,:\right)\rV_2^2+ \sum_j\sum_{1\le t\leq|\hat{\Omega}|}\chi_{[1,\infty)}(\hat{a}_t+1-j)\lV [\BDeltaRk\BSigmas^{1/2}]\left(j,:\right)\rV_2^2\Bigg)\cr
    \le &~\frac{\theta^2}{2}\left(\sum_i|\hat{\Omega}|\lV [\BDeltaLk\BSigmas^{1/2}]\left(i,:\right)\rV_2^2+\sum_j|\hat{\Omega}|\lV [\BDeltaRk\BSigmas^{1/2}]\left(j,:\right)\rV_2^2\right)\cr
    \le&~ \frac{\theta^2}{2}\left(|\widehat{\Omega}| \lV \BDeltaLk\BSigmas^{1/2}\rV_\fro^2+|\widehat{\Omega}|\lV \BDeltaRk\BSigmas^{1/2}\rV_\fro^2\right),
\end{align*}
where the first inequality follows from $\lv\widehat{\Phi}_t\rv=\varsigma_{\hat{a}_t}$ and the Cauchy-Schwarz inequality,  the third inequality follows from $(a+b)^2\le 2a^2+2b^2$, and $\chi_{[1,\infty)}$ is the indicator function of $[1,\infty)$. Therefore, by $\sqrt{a^2+b^2}\le \lv a\rv+\lv b \rv$  we conclude the inequality \eqref{ineq:PiGDelta}.
 Similarly, by
\begin{align*}
\lv\l\Pi_{\widehat{\Omega}}\G^{*}\left(\bm{L} \bm{R}^*\right),\G^{*}\left(\bm{L} \bm{R}^*\right)\r\rv
&\le\sum_{\left(i,j\right)\in \widehat{\Phi}}\lv [\bm{L}(i,:) \bm{R}^*(j,:)]^2\rv
\le \sum_{\left(i,j\right)\in \widehat{\Phi}}\lV \bm{L}(i,:)\rV_2^2\lV \bm{R}^*(j,:)\rV_2^2\cr
&\le |\widehat{\Omega}| \min\{\lV\BL\rV_{2,\infty}^2\lV\BR\rV_\fro^2,\lV\BR\rV_{2,\infty}^2\lV\BL\rV_{\fro}^2\},
\end{align*}
we conclude the second statement.
\end{proof}

\begin{lemma}\label{err:sUV}
Let $\bs = \Gamma_{\gamma\alpha }\left(\Pi_{\Omega}\left(\bf-\G^*\left(\BL_{k}\BR_{k}^*\right)\right)\right)$ with $\gamma > 1$. Then, for all $\BL_{\BA}\in \mathbb{C}^{n_1\times r}$, $\BR_{\BA}\in \mathbb{C}^{n_2\times r}$, under event \eqref{eq:projectionRIP1}, \eqref{ineq:uv} and provided $m\ge 6C_{\varepsilon_0,\kappa}^2\varepsilon_0^{-2}c_s^2 \mu^2 r^2\log n$, we have
\begin{align*}
&~\frac{1}{p}\lv\l \G\Pi_\Omega\left(\bs-\bss\right),\BL_{\BA}\BR_{\BA}^* \r \rv \cr
\le&~ \sqrt{n}\xi
\left( \lV [\BDeltaL\BSigmas^{1/2}]\rV_\fro+\lV [\BDeltaR\BSigmas^{1/2}]\rV_\fro\right)\min\{\lV\BL_{\BA}\rV_{2,\infty}\lV\BR_{\BA}\rV_\fro,\lV\BR_{\BA}\rV_{2,\infty}\lV\BL_{\BA}\rV_{\fro}\},
\end{align*}
where $\xi =\left((\gamma+1)\alpha\left(1+\frac{c\kappa}{1-\varepsilon_0}\right)\sqrt{\frac{c_s\mu r }{2}}+(1+3\sqrt{\varepsilon_0})\frac{\sqrt{(\gamma+1)\alpha}}{\sqrt{(\gamma-1)}}\right)$.
\end{lemma}
\begin{proof}
Denote $\BDelta_k: = \BL_{k}\BR_{k}^*- \G\bzs$, and $\Omega^{s}:=\supp\left(\Pi_\Omega\bs\right)$, $\Omega^{s_{\star}}:= \supp\left(\Pi_\Omega\bss\right)$. By the definition of $\bs$, we have
\begin{align}\label{iins}
\lv\left[\Pi_{\Omega}\bss-\bs\right]_i\rv
=\lv\left[\Pi_{\Omega}\left(\bzs-\G^*\left(\BL_{k}\BR_{k}^*\right)\right)\right]_i\rv=\lv\left[\Pi_{\Omega}(\G^*\BDelta_k)\right]_i\rv,~\text{for}~ i\in \Omega^{s},
\end{align}
where the last equation follows from $\G^*\G=\I$. For $i\in \Omega^{s_{\star}}\setminus\Omega^{s}$, by
$$
\bs
= \Gamma_{\gamma\alpha }\left(\Pi_{\Omega}\left(\bss-\G^*\left(\BL_{k}\BR_{k}^*-\G\bzs \right)\right)\right)
=\Gamma_{\gamma\alpha }\left(\Pi_{\Omega}\left(\bss-\G^*\BDelta_k\right)\right),
$$
we know $\lv\left[\Pi_{\Omega}\left(\bss-\G^*\BDelta_k\right)\right]_{i}\rv$ is smaller than the $\gamma \alpha m$-th largest-in-magnitude entry of $\lv\Pi_{\Omega}\left(\bss-\G^*\BDelta_k\right)\rv$. Noticing $\Pi_{\Omega}\bss$ is at most $\alpha m$ sparse, thus $\lv\left[\Pi_{\Omega}\left(\bss-\G^*\BDelta_k\right)\right]_{i}\rv$ is smaller than the $(\gamma-1) \alpha m$-th largest entry of $\lv\Pi_{\Omega}\left(\G^*\BDelta_k\right)\rv$, and $\lv\left[\Pi_{\Omega}\left(\bss-\G^*\BDelta_k\right)\right]_{i}\rv^2 \le {\frac{\lv\l \Pi_{\Omega}\left(\G^*\BDelta_k\right),\G^*\BDelta_k\r\rv}{(\gamma-1)\alpha m}}$. Then, we have
\begin{equation}\label{iinsc}
\lv\left[\Pi_{\Omega}\bss\right]_{i}\rv \le {\frac{\lv\l \Pi_{\Omega}\left(\G^*\BDelta_k\right),\G^*\BDelta_k\r\rv^{1/2}}{\sqrt{(\gamma-1)\alpha m}}}+\lv\left[\Pi_{\Omega}\left(\G^*\BDelta_k\right)\right]_{i}\rv,~\text{for}~ i\in \Omega^{s_{\star}}\setminus\Omega^{s}.
\end{equation}
Hence,
\begin{align*}
&~\lv\l \G\Pi_\Omega\left(\bs-\bss\right),\BL_{\BA}\BR_{\BA}^* \r \rv=\lv\l \Pi_\Omega\left(\bs-\bss\right),\G^*(\BL_{\BA}\BR_{\BA}^*) \r \rv\cr
\le&~\sum_{i\in \Omega^{s}}\lv\left[\Pi_\Omega\left(\bs-\bss\right)\right]_i\rv\lv \left[\G^*(\BL_{\BA}\BR_{\BA}^*)\right]_{i} \rv+\sum_{i\in \Omega^{s_{\star}}\setminus\Omega^{s}}\lv\left[\Pi_\Omega\left(\bss\right)\right]_i\rv\lv \left[\G^*(\BL_{\BA}\BR_{\BA}^*)\right]_{i} \rv\cr
\le&~\sum_{i\in\Omega^{s}\cup \Omega^{s_{\star}}}\lv\left[\Pi_{\Omega}\left(\G^*\BDelta_k\right)\right]_{i}\rv\lv \left[\G^*(\BL_{\BA}\BR_{\BA}^*)\right]_{i} \rv +\sum_{i\in \Omega^{s_{\star}}\setminus\Omega^{s}}{\frac{\lv\l \Pi_{\Omega}\left(\G^*\BDelta_k\right),\G^*\BDelta_k\r\rv^{\frac{1}{2}}}{\sqrt{(\gamma-1)\alpha m}}\lv \left[\G^*(\BL_{\BA}\BR_{\BA}^*)\right]_{i} \rv}\cr
\le &~ \lv\l\Pi_{\Omega^{s}\cup \Omega^{s_{\star}}}\left(\G^*\BDelta_k\right),\G^*\BDelta_k\r\rv^{\frac{1}{2}}\lv\l \Pi_{\Omega^{s}\cup \Omega^{s_{\star}}}\G^*(\BL_{\BA}\BR_{\BA}^*),\G^*(\BL_{\BA}\BR_{\BA}^*)\r\rv^{\frac{1}{2}}\cr
&~+\frac{\lv\l \Pi_{\Omega}\left(\G^*\BDelta_k\right),\G^*\BDelta_k\r\rv}{2\beta(\gamma-1)}+\frac{\beta}{2}\lv\l \Pi_{\Omega^{s}\cup \Omega^{s_{\star}}}\G^*(\BL_{\BA}\BR_{\BA}^*),\G^*(\BL_{\BA}\BR_{\BA}^*)\r\rv \cr
\le &~(1+\frac{c\kappa}{1-\varepsilon_0}) \sqrt{\frac{c_s\mu r }{2n}}(\gamma+1)\alpha m\left( \lV \BDeltaLk\BSigmas^{1/2}\rV_\fro+\lV \BDeltaRk\BSigmas^{1/2}\rV_\fro\right)\cr
&~\times\min\{\lV\BL_{\BA}\rV_{2,\infty}\lV\BR_{\BA}\rV_\fro,\lV\BR_{\BA}\rV_{2,\infty}\lV\BL_{\BA}\rV_{\fro}\}
+\frac{\lv\l\Pi_{\Omega}\G^*(\BDelta_k),\G^*(\BDelta_k)\r\rv}{2\beta(\gamma-1)}\cr
&~+\frac{\beta(\gamma+1)\alpha m}{2}
\min\{\lV\BL_{\BA}\rV_{2,\infty}^2\lV\BR_{\BA}\rV_\fro^2,\lV\BR_{\BA}\rV_{2,\infty}^2\lV\BL_{\BA}\rV_{\fro}^2\},
\end{align*}
where the third inequality follows form $ab\le \frac{\beta}{2}a^2+\frac{b^2}{2\beta}$ for any $\beta>0$, and the Cauchy-Schwarz inequality, the fourth inequality follows from \Cref{PZUZV-M}. Letting $$\beta =\frac{\lv\l\Pi_{\Omega}\G^*(\BDelta_k),\G^*(\BDelta_k)\r\rv^{\frac{1}{2}}}{(\sqrt{(\gamma-1)(\gamma+1)\alpha m} \min\{\lV\BL_{\BA}\rV_{2,\infty}\lV\BR_{\BA}\rV_\fro,\lV\BR_{\BA}\rV_{2,\infty}\lV\BL_{\BA}\rV_{\fro}\})},$$
 we then have
\begin{align*}
&~\lv\l \G\Pi_\Omega\left(\bs-\bss\right),\BL_{\BA}\BR_{\BA}^* \r \rv\cr
\le 
&~ \Big[\Big(1+\frac{c\kappa}{1-\varepsilon_0}\Big)\sqrt{\frac{c_s\mu r }{2n}}(\gamma+1)\alpha m\left( \lV \BDeltaLk\BSigmas^{1/2}\rV_\fro+\lV \BDeltaRk\BSigmas^{1/2}\rV_\fro\right)+\cr
&~\frac{\sqrt{(\gamma+1)\alpha m}}{\sqrt{(\gamma-1)}}\lv\l\Pi_{\Omega}\G^*(\BDelta_k),\G^*(\BDelta_k)\r\rv^{\frac{1}{2}}\Big]\min\{\lV\BL_{\BA}\rV_{2,\infty}\lV\BR_{\BA}\rV_\fro,\lV\BR_{\BA}\rV_{2,\infty}\lV\BL_{\BA}\rV_{\fro}\}\cr
\le &~ p\sqrt{n}\xi\left( \lV \BDeltaLk\BSigmas^{1/2}\rV_\fro+\lV \BDeltaRk\BSigmas^{1/2}\rV_\fro\right)\min\{\lV\BL_{\BA}\rV_{2,\infty}\lV\BR_{\BA}\rV_\fro,\lV\BR_{\BA}\rV_{2,\infty}\lV\BL_{\BA}\rV_{\fro}\},
\end{align*}
where $\xi = \left((\gamma+1)\alpha\left(1+\frac{c\kappa}{1-\varepsilon_0}\right)\sqrt{\frac{c_s\mu r }{2}}+(1+3\sqrt{\varepsilon_0})\frac{\sqrt{(\gamma+1)\alpha}}{\sqrt{(\gamma-1)}}\right)$ and the last inequality follows from \Cref{projectioner:Gdelta}.
\end{proof}

\subsubsection{Contraction Property in Full Observation} \label{subsec:prooflocalcon_full}

\begin{proof}[Proof of \Cref{thm:convergence_full}]
For the case of full observation $m=n$ ($p=1$), we show the contraction property for one step update of ${\BL}_{k}$, ${\BR}_{k}$ in terms of the metric $d_k$. Denote
\begin{align*}
\widetilde{\BL}_{k+1}&:=\BL_{k}-\eta\nabla_{\BL} \ell(\BL_{k},\BR_{k};\bs_{k+1})\left(\BR^{*}_k\BR_k\right)^{-1}\cr
&=\BL_{k}-\eta \G\left(\bs_{k+1}-\bf\right)\BR_{k}\left(\BR^{*}_{k}\BR_{k}\right)^{-1}  - \eta\BL_{k}.
\end{align*}
Then, we have
\begin{align*}
&~\left(\widetilde{\BL}_{k+1}\BQ_{k}-\BLs\right)\BSigmas^{1/2}
=\left[\BDeltaLk-\eta \G\big[\bs_{k+1}-\bzs-\bss\big]\BRn\left(\BRn^{*}\BRn\right)^{-1}  - \eta\BLn\right]\BSigmas^{1/2}\cr
=&~\BDeltaLk\BSigmas^{1/2}-\eta \big[\G\left(\bs_{k+1}-\bss\right)+\left(\BLn\BRn^{*}-\BLs\BRs^{*}\right)\big]\BRn\left(\BRn^{*}\BRn\right)^{-1}\BSigmas^{1/2}\cr
=&~\underbrace{(1-\eta)\BDeltaLk\BSigmas^{1/2}- \eta\BLs\BDeltaRk^{*}\BRn\left(\BRn^{*}\BRn\right)^{-1}\BSigmas^{1/2}}_{\bm{T}_1}
\cr
&~-\underbrace{\eta [\G\left(\bs_{k+1}-\bss\right)]\BRn\left(\BRn^{*}\BRn\right)^{-1}\BSigmas^{1/2}}_{\bm{T}_2},
\end{align*}
where in the third equation we have used the decomposition 
$\BLn\BRn^{*}-\BLs\BRs^{*}= \BDeltaLk\BRn^{*}+\BLs\BDeltaRk^{*}=\BDeltaLk\BRs^{*}+\BLs\BDeltaRk^{*}+\BDeltaLk\BDeltaRk^{*}$. Noticing that
\begin{equation}\label{eqn:dec23terms}
\lV\left(\widetilde{\BL}_{k+1}\BQ_{k}-\BLs\right)\BSigmas^{1/2}\rV_{\fro}^2
=\lV\bm{T}_1\rV_{\fro}^2
+\lV\bm{T}_2\rV_{\fro}^2-2\Re\l\bm{T}_1,\bm{T}_2\r,
\end{equation}
 and we have the following claims, whose proof is deferred later in this subsection:
 \begin{claim}\label{claim:Tboundfull}
For $(\BL_{k},\BR_{k})\in \mathcal{E}(\varepsilon_0 \sigmas_r,c \sqrt{\frac{c_s \mu r}{n}}\sigmas_1)$, we have
 \begin{align*}
  \lV \bm{T}_1\rV_{\fro}^2
\le &~ \left((1-\eta)^2+\frac{2\varepsilon_0}{1-\varepsilon_0}\eta(1-\eta)+\frac{2\varepsilon_0+\varepsilon_0^2}{(1-\varepsilon_0)^2}\eta^2\right)d_k^2 ,\cr
\lv\l\bm{T}_1,\bm{T}_2\r\rv
\le  & ~ \frac{(\eta-\eta^2)c\kappa\sqrt{c_s\mu r}\xi}{2(1-\varepsilon_0)^3}
\left(3\lV \BDeltaLk\BSigmas^{1/2}\rV_\fro^2+\lV \BDeltaRk\BSigmas^{1/2}\rV_\fro^2\right)\cr
&~+\frac{\eta^2\sqrt{c_s\mu r}\xi}{2(1-\varepsilon_0)^2}\left(\lV \BDeltaLk\BSigmas^{1/2}\rV_\fro^2+3\lV \BDeltaRk\BSigmas^{1/2}\rV_\fro^2\right), \cr
\lV\bm{T}_2\rV_{\fro}
\le&~ \frac{2\eta c\kappa\sqrt{c_s\mu r}\xi}{(1-\varepsilon_0)^3}\left( \lV \BDeltaLk\BSigmas^{1/2}\rV_\fro+\lV \BDeltaRk\BSigmas^{1/2}\rV_\fro\right).
 \end{align*}
 \end{claim}
Recall that $\xi =\left((\gamma+1)\alpha\left(1+\frac{c\kappa}{1-\varepsilon_0}\right)\sqrt{\frac{c_s\mu r }{2}}+(1+3\sqrt{\varepsilon_0})\frac{\sqrt{(\gamma+1)\alpha}}{\sqrt{(\gamma-1)}}\right)$. Therefore,  provided $\alpha \le \frac{\varepsilon^2_0(\gamma-1)}{64(\gamma+1)(1+c\kappa)^2c_s\mu r}$, by substituting the above three inequalities in \Cref{claim:Tboundfull} to \eqref{eqn:dec23terms}, we have
\begin{align*}
&\quad\lV\left(\widetilde{\BL}_{k+1}\BQ_{k}-\BLs\right)\BSigmas^{1/2}\rV_{\fro}^2 \cr
&\le C_1(\varepsilon_0,\eta)d_k^2+ C_2(\varepsilon_0,\eta)\lV \BDeltaLk\BSigmas^{1/2}\rV_\fro^2+ C_3(\varepsilon_0,\eta)\lV \BDeltaRk\BSigmas^{1/2}\rV_\fro^2\cr
&\le \left(1-0.6\eta\right)^2d_k^2,\quad \text{for}~ \varepsilon_0=0.01,~ 0< \eta \le 1,
\end{align*}
where we have used the Cauchy-Schwarz inequality and 
and the constants satisfy  $C_1(\varepsilon_0,\eta)\le (1-\eta)^2+\frac{2\varepsilon_0}{1-\varepsilon_0}\eta(1-\eta)+\frac{8\eta^2 \varepsilon_0^2}{(1-\varepsilon_0)^6}$,  $C_2(\varepsilon_0,\eta)\le \frac{3\eta\varepsilon_0-2\eta^2 \varepsilon_0-\eta^2\varepsilon_0^2}{2(1-\varepsilon_0)^3}$,
$C_3(\varepsilon_0,\eta)\le\frac{\eta\varepsilon_0+2\eta^2\varepsilon_0-3\eta^2\varepsilon_0^2}{2(1-\varepsilon_0)^3}$. Similarly, let $\widetilde{\BR}_{k+1}:=\BR_{k}-\eta\nabla_{\BR} \ell(\BL_{k},\BR_{k};\bs_{k+1})\left(\BR^{*}_k\BR_k\right)^{-1},$
we can show $
\lV\left(\widetilde{\BR}_{k+1}\BQ_{k}^{-*}-\BRs\right)\BSigmas^{1/2}\rV_{\fro}^2\le (1-0.6\eta)^2 d_k^2$.
Thus we have $$\mathrm{dist}\left(\widetilde{\BL}_{k+1},\widetilde{\BR}_{k+1};\BLs,\BRs\right) \le \left(1-0.6\eta\right)d_k.$$
Then, \Cref{PC:Non-expansiveness} gives 
$ d_{k+1}\le (1-0.6\eta)d_k$ and $(\BL_{k+1},\BR_{k+1})\in \mathcal{E}(\varepsilon_0 \sigmas_r,c \sqrt{\frac{c_s \mu r}{n}}\sigmas_1)$. Finally, we will complete the proof by proving \Cref{claim:Tboundfull}. 
\end{proof}
\begin{proof}[Proof of \Cref{claim:Tboundfull}]
The upper bound of $\lV\bm{T}_1\rV_{\fro}$ follows the argument as in {\cite[Eq.~(46)]{tong2021accelerating}}. For $\lv\l\bm{T}_1,\bm{T}_2\r\rv$, we denote
\begin{equation}\label{def:LA1}
\begin{split}
\BL_{\BA_1}&:= \BDeltaLk\BSigmas\left(\BRn^{*}\BRn\right)^{-1}\BSigmas^{1/2},~\BR_{\BA_1}:= \BRn\BSigmas^{-1/2},\cr
\BL_{\BA_2}&:= \BLs\BSigmas^{-1/2}=\BU_{\star}, ~\BR_{\BA_2}:=\BRn\left(\BRn^{*}\BRn\right)^{-1}\BSigmas\left(\BRn^{*}\BRn\right)^{-1}\BRn^*\BDeltaRk\BSigmas^{1/2}.
\end{split}
\end{equation}
Then, by applying \Cref{err:sUV} with $m=n$ we have
\begin{align*}
&~\lv\l\bm{T}_1,\bm{T}_2\r\rv
= \Big\lvert\eta(1-\eta) \tr\left([\G\left(\bs_{k+1}-\bss\right)]\BRn\left(\BRn^{*}\BRn\right)^{-1}\BSigmas\BDeltaLk^*\right) \cr
&~\qquad\quad -\eta^2 \tr\left([\G\left(\bs_{k+1}-\bss\right)]\BRn\left(\BRn^{*}\BRn\right)^{-1}\BSigmas\left(\BRn^{*}\BRn\right)^{-1}\BRn^*\BDeltaRk\BLs^{*}\right) \Big\rvert\cr
\le &~\eta(1-\eta) \lv\l\G\left(\bs_{k+1}-\bss\right),\BL_{\BA_1}\BR_{\BA_1}^*\r\rv
+\eta^2 \lv\l\G\left(\bs_{k+1}-\bss\right),\BL_{\BA_2}\BR_{\BA_2}^*\r \rv \cr
\le & ~ \sqrt{n}\xi
\left( \lV \BDeltaLk\BSigmas^{1/2}\rV_\fro+\lV \BDeltaRk\BSigmas^{1/2}\rV_\fro\right)\times\cr
&~\left((\eta-\eta^2)\lV\BR_{\BA_1}\rV_{2,\infty}\lV\BL_{\BA_1}\rV_\fro+\eta^2\lV\BL_{\BA_2}\rV_{2,\infty}\lV\BR_{\BA_2}\rV_\fro\right),
\end{align*}
together with $ab\le\frac{a^2+b^2}{2}$, \eqref{boundnorm:c}, \eqref{boundnorm:d} and the fact that
\begin{align}\label{ineq:DeltaL2norm}
\lV\BSigmas^{1/2}\left(\BRn^{*}\BRn\right)^{-1}\BSigmas^{1/2}\rV_2&=\lV\BRn\left(\BRn^{*}\BRn\right)^{-1} \BSigmas^{1/2}\rV_2^2,\cr
\lV\BRn\left(\BRn^{*}\BRn\right)^{-1}\BSigmas\left(\BRn^{*}\BRn\right)^{-1}\BRn^*\rV_2&=\lV\BRn\left(\BRn^{*}\BRn\right)^{-1} \BSigmas^{1/2}\rV_2^2,\cr
~~~~\mathop{\Longrightarrow}\limits_{\eqref{boundnorm:d}}\lV\BL_{\BA_1}\rV_\fro\le \frac{1}{(1-\varepsilon_0)^2}\lV\BDeltaLk\BSigmas^{1/2}\rV_\fro,&~\lV\BR_{\BA_2}\rV_\fro\le \frac{1}{(1-\varepsilon_0)^2}\lV\BDeltaLk\BSigmas^{1/2}\rV_\fro,
\end{align}
we conclude the second inequality in \Cref{claim:Tboundfull}. For $\lV\bm{T}_2\rV_{\fro}$, using the variational representation of the Frobenius norm, for some $\BL_{\BA}\in \mathbb{C}^{n_1\times r}$, $\lV\BL_{\BA}\rV_{\fro}=1$ we have
 \begin{align*}
 &\quad\lV\bm{T}_2\rV_{\fro}=\eta \lv\l[\G\left(\bs_{k+1}-\bss\right)]\BRn\left(\BRn^{*}\BRn\right)^{-1}\BSigmas^{1/2},\BL_{\BA}\r\rv\cr
 &=\eta \lv\l[\G\left(\bs_{k+1}-\bss\right)],\BL_{\BA}\BSigmas^{1/2}\left(\BRn^{*}\BRn\right)^{-1}\BSigmas^{1/2}(\BRn\BSigmas^{-1/2})^*\r\rv\cr
 &\le \eta \sqrt{n}\xi
\left( \lV \BDeltaLk\BSigmas^{1/2}\rV_\fro+\lV \BDeltaRk\BSigmas^{1/2}\rV_\fro\right)\lV\BRn\BSigmas^{-\frac{1}{2}}\rV_{2,\infty}\lV\BL_{\BA}\BSigmas^{1/2}\left(\BRn^{*}\BRn\right)^{-1}\BSigmas^{1/2}\rV_\fro\cr
&\le \frac{2\eta c\kappa\sqrt{c_s\mu r}\xi}{(1-\varepsilon_0)^3}\left( \lV \BDeltaLk\BSigmas^{1/2}\rV_\fro+\lV \BDeltaRk\BSigmas^{1/2}\rV_\fro\right),
\end{align*}
where the first inequality follows from \Cref{err:sUV}, the last inequality follows from \eqref{boundnorm:c} and similar argument to \eqref{ineq:DeltaL2norm}. 
\end{proof}

\subsubsection{Contraction Property in Partial Observation} \label{proofthm:convergence_partial}
The following lemma is crucial for the proof of the theorem.
\begin{lemma}\label{projectionerr2}
For any $\BL_{\BA}, \BL_{\BB}\in\mathbb{C}^{n_1\times r}$,  $\BR_{\BA}, \BR_{\BB}\in\mathbb{C}^{n_2\times r}$ independent of $\Omega$,  with probability at least $1-2n^{-2}$ it holds that
\begin{align*}
&~\lv\l\G\left(\frac{1}{p}\Pi_\Omega-\I\right)\G^*  \left(\BL_{\BA}\BR_{\BA}^*\right),\BL_{\BB}\BR_{\BB}^*\r\rv\cr
\le&~\sqrt{ \frac{24 n\log n}{p}} \min\{\lV\BL_{\BA}\rV_{2,\infty}\lV\BL_{\BB}\rV_\fro,
\lV\BL_{\BA}\rV_{\fro}\lV\BL_{\BB}\rV_{2,\infty}\}\min\{\lV\BR_{\BA}\rV_{2,\infty}\lV\BR_{\BB}\rV_\fro,\lV\BR_{\BA}\rV_{\fro}\lV\BR_{\BB}\rV_{2,\infty}\}
\end{align*}
provided $m\ge 24 \log n$.
\end{lemma}
\begin{proof}
First, we have
\begin{align*}
&~\lv\l\G\left(\frac{1}{p}\Pi_\Omega-\I\right)\G^*  \left(\BL_{\BA}\BR_{\BA}^*\right),\BL_{\BB}\BR_{\BB}^*\r\rv\cr
=&~ \frac{1}{p}\lv\l\left(\Pi_\Omega-p\I\right)\G^*  \left(\BL_{\BA}\BR_{\BA}^*\right),\G^{*}\left(\BL_{\BB}\BR_{\BB}^*\right)\r\rv=\frac{1}{p} \lv\sum_{k=1}^m r_{k}\rv,
\end{align*}
where 
$$r_{k}:=\left[\G^*  \left(\BL_{\BA}\BR_{\BA}^*\right)\right]_{a_k}\left[\G^*  \left(\BL_{\BB}\BR_{\BB}^*\right)\right]_{a_k}-n^{-1}\l\G^{*}\left(\BL_{\BA}\BR_{\BA}^*\right),\G^{*}\left(\BL_{\BB}\BR_{\BB}^*\right)\r, ~k\in [m],$$
then we have $\mathbb{E}[r_{k}]=0$. Further, noticing that for any $a\in [n]$, we have
\begin{align*}
&~\lv\left[\G^*  \left(\BL_{\BA}\BR_{\BA}^*\right)\right]_{a}\left[\G^*  \left(\BL_{\BB}\BR_{\BB}^*\right)\right]_{a}\rv\cr
=&~\lv\left(\sum_{(i,j)\in\Omega_{a}}\frac{1}{\varsigma_{a}}\BL_{\BA}(i,:)\BR_{\BA}(j,:)^*\right)\left(\sum_{(i,j)\in\Omega_{a}} \BL_{\BB}(i,:)\BR_{\BB}(j,:)^*\right)\rv\cr
\le &~ \left(\frac{1}{\varsigma_a}\sum_{(i,j)\in\Omega_a} \lV\BL_{\BA}(i,:)\rV_2\lV\BR_{\BA}(j,:)\rV_2\right)\left(\sum_{(i,j)\in\Omega_a} \lV\BL_{\BB}(i,:)\rV_2\lV\BR_{\BB}(j,:)\rV_2\right)\cr
\le &~ \lV\BL_{\BA}\rV_{2,\infty}\lV\BR_{\BA}\rV_{2,\infty}\lV\BL_{\BB}\rV_{\fro}\lV\BR_{\BB}\rV_{\fro},
\end{align*}
where the second and the inequalities follows from the Cauchy-Schwarz inequality and the fact that $\lv\Omega_a\rv= \varsigma_a$. Thus
\begin{align*}
\lv r_{k}\rv
\le  2\lV\BL_{\BA}\rV_{2,\infty}\lV\BR_{\BA}\rV_{2,\infty}\lV\BL_{\BB}\rV_{\fro}\lV\BR_{\BB}\rV_{\fro}.
\end{align*}
 Similarly, we have
$
\lv r_{k}\rv \le 2\lV\BL_{\BA}\rV_{\fro}\lV\BR_{\BA}\rV_{\fro}\lV\BL_{\BB}\rV_{2,\infty}\lV\BR_{\BB}\rV_{2,\infty}.
$
Also, for any $a\in [n]$, by Cauchy-Schwarz inequality and $\lv\Omega_a\rv= \varsigma_a$ we have
\begin{align*}
&~\lv\left[\G^*  \left(\BL_{\BA}\BR_{\BA}^*\right)\right]_{a}\left[\G^*  \left(\BL_{\BB}\BR_{\BB}^*\right)\right]_{a}\rv\cr
\le&~ \left(\sum_{(i,j)\in\Omega_a}\frac{1}{\sqrt{\varsigma_a}}\lV\BL_{\BA}(i,:)\rV_{2}\lV\BR_{\BA}(j,:)\rV_{2}\right)\left(\frac{1}{\sqrt{\varsigma_a}}\sum_{(i,j)\in\Omega_a} \lV\BL_{\BB}(i,:)\rV_2\lV\BR_{\BB}(j,:)\rV_2\right)\cr
\le & \lV\BL_{\BA}\rV_{2,\infty}\lV\BL_{\BB}\rV_{2,\infty}\left(\sum_{(i,j)\in\Omega_a}\frac{1}{\sqrt{\varsigma_a}}\lV\BR_{\BA}(j,:)\rV_{2}\right)\left(\frac{1}{\sqrt{\varsigma_a}}\sum_{(i,j)\in\Omega_a} \lV\BR_{\BB}(j,:)\rV_2\right)\cr
\le & \lV\BL_{\BA}\rV_{2,\infty}\lV\BL_{\BB}\rV_{2,\infty}\sqrt{ \sum_{i=1}^a\lV\BR_{\BA}(i,:)\rV_2^2}\sqrt{ \sum_{j=1}^a\lV\BR_{\BB}(j,:)\rV_2^2 }\cr
\le & \lV\BL_{\BA}\rV_{2,\infty}\lV\BL_{\BB}\rV_{2,\infty}\lV\BR_{\BA}\rV_{\fro}\lV\BR_{\BB}\rV_{\fro},
\end{align*}
Then we have $\lv r_{k}\rv \le 2\lV\BL_{\BA}\rV_{2,\infty}\lV\BL_{\BB}\rV_{2,\infty}\lV\BR_{\BA}\rV_{\fro}\lV\BR_{\BB}\rV_{\fro}$. Similarly,
$$
\lv r_{k}\rv\le  2\lV\BL_{\BA}\rV_{\fro}\lV\BL_{\BB}\rV_{\fro}\lV\BR_{\BA}\rV_{2,\infty}\lV\BR_{\BB}\rV_{2,\infty}.
$$
Denote
\begin{equation*}
\begin{split}
\zeta:=\min\{\lV\BL_{\BA}\rV_{2,\infty}\lV\BL_{\BB}\rV_\fro,\lV\BL_{\BA}\rV_{\fro}\lV\BL_{\BB}\rV_{2,\infty}\} \min\{\lV\BR_{\BA}\rV_{2,\infty}\lV\BR_{\BB}\rV_\fro,\lV\BR_{\BA}\rV_{\fro}\lV\BR_{\BB}\rV_{2,\infty}\} 
\end{split}
\end{equation*} 
Then, combining all the pieces above we obtain 
$$\lv r_{k}\rv\le 2\zeta.$$
By using the above bounds for $\lv\left[\G^*  \left(\BL_{\BA}\BR_{\BA}^*\right)\right]_{a}\left[\G^*  \left(\BL_{\BB}\BR_{\BB}^*\right)\right]_{a}\rv$ and the inequality $(x+y)^2\le 2x^2+2y^2$, we further have
\begin{align*}
\sum_{k=1}^m\mathbb{E}[r_{k}^2]\le \sum_{k=1}^m 4\max_{a\in [n]}\lv\left[\G^*  \left(\BL_{\BA}\BR_{\BA}^*\right)\right]_{a}\left[\G^*  \left(\BL_{\BB}\BR_{\BB}^*\right)\right]_{a}\rv^2\le  4m \zeta^2.
\end{align*}
By Bernstein's inequality, it then implies
\begin{align*}
\mathbb{P}\left(\lv\sum_{k=1}^m r_{k}\rv\ge t\right)
\le  2\exp\left(\frac{-t^2/2}{4 m\zeta^2+\frac{4t}{3}\zeta}\right).
\end{align*}
Letting $t=\sqrt{24 np\log n}\zeta$, by the definition of $r_{k}$ we then have
$$\lv\l\G\left(\frac{1}{p}\Pi_\Omega-\I\right)\G^*  \left(\BL_{\BA}\BR_{\BA}^*\right),\BL_{\BB}\BR_{\BB}^*\r\rv\le \frac{t}{p}=\sqrt{\frac{24 n\log n}{p}}\zeta $$
with probability at least $1-2n^{-2}$ provided $p\ge (24\log n)/n$.
\end{proof}


\begin{proof}[Proof of \Cref{thm:convergence}]
Similar to the proof of \Cref{thm:convergence_full}, let $$\widetilde{\BL}_{k+1}:=\BL_{k}-\eta\nabla_{\BL} \ell(\BL_{k},\BR_{k};\bs_{k+1})\left(\BR^{*}_k\BR_k\right)^{-1}$$ we have
\begin{small}
\begin{align*}
&\left(\widetilde{\BL}_{k+1}\BQ_{k}-\BLs\right)\BSigmas^{1/2}\cr
=&\left[\BDeltaLk-\eta \G\big[\frac{1}{p}\Pi_\Omega\left(\G^{*}\left(\BLn\BRn^{*}\right)+\bs_{k+1}-\bf\right)-\G^{*}(\BLn\BRn^{*})\big]\BRn\left(\BRn^{*}\BRn\right)^{-1}  - \eta\BLn\right]\BSigmas^{1/2}\cr
=&\BDeltaLk\BSigmas^{1/2}-\eta \left(\frac{1}{p}\G\Pi_\Omega\G^{*}-\G\G^{*}\right)\left(\BLn\BRn^{*}-\BLs\BRs^{*}\right)\BRn\left(\BRn^{*}\BRn\right)^{-1}\BSigmas^{1/2} \cr
&~-\frac{\eta}{p}[\G\Pi_\Omega\left(\bs_{k+1}-\bss\right)]\BRn\left(\BRn^{*}\BRn\right)^{-1}\BSigmas^{1/2}- \eta\left(\BLn\BRn^{*}-\BLs\BRs^{*}\right)\BRn\left(\BRn^{*}\BRn\right)^{-1}\BSigmas^{1/2}\cr
=&(1-\eta)\BDeltaLk\BSigmas^{1/2}-\eta \left(\frac{1}{p}\G\Pi_\Omega\G^{*}-\G\G^{*}\right)\left(\BDeltaLk\BRs^{*}+\BLn\BDeltaRk^{*}\right)\BRn\left(\BRn^{*}\BRn\right)^{-1}\BSigmas^{1/2} \cr
&~-\frac{\eta}{p}[\G\Pi_\Omega\left(\bs_{k+1}-\bss\right)]\BRn\left(\BRn^{*}\BRn\right)^{-1}\BSigmas^{1/2} -\eta\BLs\BDeltaRk^{*}\BRn\left(\BRn^{*}\BRn\right)^{-1}\BSigmas^{1/2}\cr
=&\underbrace{(1-\eta)\BDeltaLk\BSigmas^{1/2}- \eta\BLs\BDeltaRk^{*}\BRn\left(\BRn^{*}\BRn\right)^{-1}\BSigmas^{1/2}}_{\bm{T}_1}
\cr
&~-\underbrace{ \frac{\eta}{p}[\G\Pi_\Omega\left(\bs_{k+1}-\bss\right)]\BRn\left(\BRn^{*}\BRn\right)^{-1}\BSigmas^{1/2}}_{\bm{T}_3}\cr
&~-\underbrace{\eta \left(\frac{1}{p}\G\Pi_\Omega\G^{*}-\G\G^{*}\right)\left(\BDeltaLk\BRs^{*}+\BLn\BDeltaRk^{*}\right)\BRn\left(\BRn^{*}\BRn\right)^{-1}\BSigmas^{1/2}}_{\bm{T}_4},
\end{align*}
\end{small}%
where the second equation follows from $\bf=\bzs+\bss$, $\G^{*}(\BLs\BRs^{*})=\bzs$ and $\G\G^{*}(\BLs\BRs^{*})=\G\bzs=\BLs\BRs^{*}$, the third equation follows from $\BLn\BRn^{*}-\BLs\BRs^{*}= \BDeltaLk\BRn^{*}+\BLs\BDeltaRk^{*}=\BDeltaLk\BRs^{*}+\BLs\BDeltaRk^{*}+\BDeltaLk\BDeltaRk^{*}$.
Noticing that
\begin{align}\label{eq:termsplit}
&~\lV\left(\widetilde{\BL}_{k+1}\BQ_{k}-\BLs\right)\BSigmas^{1/2}\rV_{\fro}^2\cr
=&~\lV\bm{T}_1\rV_{\fro}^2
+\lV\bm{T}_3\rV_{\fro}^2+\lV\bm{T}_4\rV_{\fro}^2-2\Re\l\bm{T}_1,\bm{T}_3\r-2\Re\l\bm{T}_1,\bm{T}_4\r+2\Re\l\bm{T}_3,\bm{T}_4\r,
\end{align}
thus we bound $\lV\left(\widetilde{\BL}_{k+1}\BQ_{k}-\BLs\right)\BSigmas^{1/2}\rV_{\fro}$ term by term by the following claims:
\begin{claim}\label{claim:Tboundpartial}
 For $(\BL_{k},\BR_{k})\in \mathcal{E}(\varepsilon_0 \sigmas_r,c \sqrt{\frac{c_s \mu r}{n}}\sigmas_1)$, $c\ge 1+\varepsilon_0$, $\varepsilon_0,\eta\in(0,1)$, provided $m\ge 24c^2(1+c\kappa)^2\varepsilon_0^{-2}c_s^2\mu^2r^2 \log n$, we have
 \begin{align*}
\lv\l\bm{T}_1,\bm{T}_3\r\rv
&\le  \frac{(\eta-\eta^2)c\kappa\sqrt{c_s\mu r}\xi}{2(1-\varepsilon_0)^3}
\left(3\lV \BDeltaLk\BSigmas^{1/2}\rV_\fro^2+\lV \BDeltaRk\BSigmas^{1/2}\rV_\fro^2\right)\cr
&\quad~+\frac{\eta^2\sqrt{c_s\mu r}\xi}{2(1-\varepsilon_0)^2}\left(\lV \BDeltaLk\BSigmas^{1/2}\rV_\fro^2+3\lV \BDeltaRk\BSigmas^{1/2}\rV_\fro^2\right),\cr
\lV\bm{T}_3\rV_{\fro}
&\le \frac{2\eta c\kappa\sqrt{c_s\mu r}\xi}{(1-\varepsilon_0)^3}\left( \lV \BDeltaLk\BSigmas^{1/2}\rV_\fro+\lV \BDeltaRk\BSigmas^{1/2}\rV_\fro\right),\cr
\lv\l\bm{T}_1,\bm{T}_4\r\rv
& \le  \left(\frac{(1-\eta)\eta \varepsilon_0}{1-\varepsilon_0}+\frac{\varepsilon_0\eta(1+\varepsilon_0)}{(1-\varepsilon_0)^4}\right)\lV\BDeltaLk\BSigmas^{1/2}\rV_{\fro}^2 \cr
&\quad~+ \left(\frac{\varepsilon_0\eta^2}{(1-\varepsilon_0)^2}+\frac{\varepsilon_0\eta(1+\varepsilon_0)}{(1-\varepsilon_0)^4}\right)\lV\BSigmas^{1/2} \BDeltaRk^{*}\rV_{\fro}^2,\cr
\lV\bm{T}_4\rV_{\fro}&\le \frac{\varepsilon_0\eta}{(1-\varepsilon_0)^2}\lV\BDeltaLk\BSigmas^{1/2}\rV_{\fro}+\frac{2\eta(1+c\kappa)}{(1-\varepsilon_0)^3}\sqrt{\frac{24c_s^2\mu^2r^2 \log n}{m}}\lV \BDeltaRk\BSigmas^{1/2}\rV_{\fro}.
\end{align*}
\end{claim}
Therefore,  for $c\ge 1+\varepsilon_0$, $\varepsilon_0,\eta\in(0,1)$,  provided $m\ge 24c^2(1+c\kappa)^2\varepsilon_0^{-2}c_s^2\mu^2r^2 \log n$ and $\alpha \le \frac{\varepsilon^2_0(\gamma-1)}{64(\gamma+1)(1+c\kappa)^2c_s\mu r}$, by combing all the pieces above together with $\lv \l\bm{T}_3,\bm{T}_4\r\rv \le \lV \bm{T}_3\rV_{\fro}\lV \bm{T}_4\rV_{\fro}$ and \eqref{eq:termsplit} we have
\begin{small}
\begin{align*}
\lV\left(\widetilde{\BL}_{k+1}\BQ_{k}-\BLs\right)\BSigmas^{1/2}\rV_{\fro}^2
&\le C_1(\varepsilon_0,\eta)d_k^2+ C_2(\varepsilon_0,\eta)\lV \BDeltaLk\BSigmas^{1/2}\rV_\fro^2+ C_3(\varepsilon_0,\eta)\lV \BDeltaRk\BSigmas^{1/2}\rV_\fro^2\cr
&\le \left(1-0.6\eta\right)^2d_k^2,\quad \text{for}~ \varepsilon_0=0.01,~ 0< \eta \le 0.6,
\end{align*}
\end{small}%
where in the first inequality we have used the Cauchy-Schwarz inequality, and the constants satisfy  $C_1(\varepsilon_0,\eta)\le (1-\eta)^2+\frac{2\varepsilon_0}{1-\varepsilon_0}\eta(1-\eta)+\frac{2\varepsilon_0+\varepsilon_0^2}{(1-\varepsilon_0)^2}\eta^2~+\frac{8\eta^2 \varepsilon_0^2}{(1-\varepsilon_0)^8}$,  $C_2(\varepsilon_0,\eta)\le \frac{\varepsilon_0^2\eta^2+8\varepsilon_0\eta+2\varepsilon_0^2\eta-5\varepsilon_0\eta^2}{(1-\varepsilon_0)^4}+\frac{2(1-\eta)\eta \varepsilon_0}{1-\varepsilon_0}+\frac{8\eta^2 \varepsilon_0^2}{(1-\varepsilon_0)^7}$,
$C_3(\varepsilon_0,\eta)\le\frac{4\eta\varepsilon_0+\eta^2\varepsilon_0+2\eta\varepsilon_0^2-3\eta^2\varepsilon_0^2}{(1-\varepsilon_0)^4}+\frac{2\varepsilon_0\eta^2}{(1-\varepsilon_0)^2}
+\frac{12\eta^2 \varepsilon_0^2}{(1-\varepsilon_0)^7}$.
Similarly, let $\widetilde{\BR}_{k+1}:=\BR_{k}-\eta\nabla_{\BR} \ell(\BL_{k},\BR_{k};\bs_{k+1})\left(\BR^{*}_k\BR_k\right)^{-1},$
we can use the same strategy to show $\|(\widetilde{\BR}_{k+1}\BQ_{k}^{-*}-\BRs)\BSigmas^{1/2}\|_{\fro}^2\le (1-0.6\eta)^2 d_k^2$.
Thus we have
\begin{equation*}
\mathrm{dist}\left(\widetilde{\BL}_{k+1},\widetilde{\BR}_{k+1};\BLs,\BRs\right) \le \left(1-0.6\eta\right)d_k.
 \end{equation*}
 By \Cref{PC:Non-expansiveness}, we conclude the statement in \Cref{thm:convergence}. Finally, we finish the proof by proving \Cref{claim:Tboundpartial}.
 \end{proof}

\begin{proof}[Proof of \Cref{claim:Tboundpartial}]
The proof of the upper bound for
$\lv \l\bm{T}_1,\bm{T}_3\r\rv$ and $\lV\bm{T}_3\rV_{\fro}$ are similar to $\lv \l\bm{T}_1,\bm{T}_2\r\rv$ and $\lV\bm{T}_2\rV_{\fro}$ as in \Cref{claim:Tboundfull}, respectively.
For $\lv \l\bm{T}_1,\bm{T}_4\r\rv$, we have:
\begin{small}
\begin{align*}
&~\l\bm{T}_1,\bm{T}_4\r 
= \l(1-\eta)\BDeltaLk\BSigmas^{1/2},\eta \left(\frac{1}{p}\G\Pi_\Omega\G^{*}-\G\G^{*}\right)\left(\BDeltaLk\BRs^{*}\right)\BRn\left(\BRn^{*}\BRn\right)^{-1}\BSigmas^{1/2}\r\cr
&~+\l(1-\eta)\BDeltaLk\BSigmas^{1/2},\eta \left(\frac{1}{p}\G\Pi_\Omega\G^{*}-\G\G^{*}\right)\left(\BLn\BDeltaRk^{*}\right)\BRn\left(\BRn^{*}\BRn\right)^{-1}\BSigmas^{1/2}\r\cr
&~-\l\eta\BLs\BDeltaRk^{*}\BRn\left(\BRn^{*}\BRn\right)^{-1}\BSigmas^{1/2},\eta \left(\frac{1}{p}\G\Pi_\Omega\G^{*}-\G\G^{*}\right)\left(\BDeltaLk\BRs^{*}\right)\BRn\left(\BRn^{*}\BRn\right)^{-1}\BSigmas^{1/2}\r\cr
&~-\l\eta\BLs\BDeltaRk^{*}\BRn\left(\BRn^{*}\BRn\right)^{-1}\BSigmas^{1/2},\eta \left(\frac{1}{p}\G\Pi_\Omega\G^{*}-\G\G^{*}\right)\left(\BLn\BDeltaRk^{*}\right)\BRn\left(\BRn^{*}\BRn\right)^{-1}\BSigmas^{1/2}\r
\end{align*}
\end{small}%
Noticing $\BRn= \BRs+\BDeltaRk$, let
\begin{equation}\label{def:LA3}
\begin{split}
\BL_{\BA_3}&:= \BDeltaLk\BSigmas^{1/2},~\BR_{\BA_3}:=\BRs\BSigmas^{-1/2},\cr
\BL_{\BA_4}&:= \BL_{\BA_3},~ \BR_{\BA_4}:=  \BDeltaRk\left(\BRn^{*}\BRn\right)^{-1}\BSigmas^{1/2},
\end{split}
\end{equation}
by \Cref{projectionerr1,projectionerr2} we then have
\begin{small}
\begin{align*}
 &~\lv\l(1-\eta)\BDeltaLk\BSigmas^{1/2},\eta \left(\frac{1}{p}\G\Pi_\Omega\G^{*}-\G\G^{*}\right)\left(\BDeltaLk\BRs^{*}\right)\left(\BRs+\BDeltaRk\right)\left(\BRn^{*}\BRn\right)^{-1}\BSigmas^{1/2}\r\rv\cr
 \le&~(1-\eta)\eta \varepsilon_0\lV\BDeltaLk\BSigmas\left(\BRn^{*}\BRn\right)^{-1} \BRs^{*}\rV_{\fro}\lV\BDeltaLk\BRs^{*} \rV_{\fro}\cr
 &~+(1-\eta)\eta \lv\l \left(\frac{1}{p}\G\Pi_\Omega\G^{*}-\G\G^{*}\right)\left(\BDeltaLk\BSigmas^{1/2}(\BRs\BSigmas^{-1/2})^{*}\right),\BDeltaLk\BSigmas\left(\BRn^{*}\BRn\right)^{-1} \BDeltaRk^{*} \r\rv\cr
 \le &~\frac{(1-\eta)\eta \varepsilon_0}{1-\varepsilon_0}\lV\BDeltaLk\BSigmas^{1/2}\rV_{\fro}^2
 +(1-\eta)\eta\sqrt{\frac{24 n\log n}{p}} \lV\BL_{\BA_3}\rV_{\fro}\lV\BR_{\BA_3}\rV_{2,\infty} \lV\BL_{\BA_4}\rV_{2,\infty} \lV\BR_{\BA_4}\rV_{\fro}\cr
 \le &~\frac{(1-\eta)\eta \varepsilon_0}{1-\varepsilon_0}\lV\BDeltaLk\BSigmas^{\frac{1}{2}}\rV_{\fro}^2+       \frac{(1+c)(\eta-\eta^2)}{(1-\varepsilon_0)^3}    \sqrt{\frac{24 c_s^2\mu^2\kappa^2 r^2\log n}{m}} \lV\BDeltaLk\BSigmas^{\frac{1}{2}}\rV_{\fro}\lV\BDeltaRk\BSigmas^{\frac{1}{2}}\rV_{\fro},
\end{align*}
\end{small}%
where the last step follows from \eqref{boundnorm:d}, \eqref{ineq:DeltaLl2inf}, $\lV\BR_{\BA_3}\rV_{2,\infty} = \lV\BU_{\star}\rV_{2,\infty}$, and \[  \lV\BR_{\BA_4}\rV_{\fro}\le  \|\BDeltaRk\BSigmas^{1/2}\|_{\fro} \|\BSigmas^{-1}\|_{2} \|\BSigmas^{1/2}(\BRn^{*}\BRn)^{-1}\BSigmas^{1/2}\|_{2}\le \frac{1}{\sigmas_r(1-\varepsilon_0)^2}\|\BDeltaRk\BSigmas^{1/2}\|_{\fro}.\]
 Similarly, by \Cref{projectionerr2} and the definition of $\BL_{\BA_2}$, $\BL_{\BA_3}$, $\BR_{\BA_2}$,  $\BR_{\BA_4}$ in \eqref{def:LA1} and \eqref{def:LA3}, we have
\begin{small}
\begin{align*}
&~\lv\l\eta\BLs\BDeltaRk^{*}\BRn\left(\BRn^{*}\BRn\right)^{-1}\BSigmas^{1/2},\eta \left(\frac{1}{p}\G\Pi_\Omega\G^{*}-\G\G^{*}\right)\left(\BDeltaLk\BRs^{*}\right)\BRn\left(\BRn^{*}\BRn\right)^{-1}\BSigmas^{1/2}\r\rv\cr
=&~\lv\l\eta(\BLs\BSigmas^{-\frac{1}{2}})\BSigmas^{\frac{1}{2}}\BDeltaRk^{*}\BRn\left(\BRn^{*}\BRn\right)^{-1}\BSigmas\left(\BRn^{*}\BRn\right)^{-1}\BRn^*,\eta \left(\frac{1}{p}\G\Pi_\Omega\G^{*}-\G\G^{*}\right)\left(\BDeltaLk\BRs^{*}\right)\r\rv\cr
\le &~ \eta^2\sqrt{\frac{24 n\log n}{p}} \lV\BL_{\BA_2}\rV_{2,\infty} \lV\BL_{\BA_3}\rV_{\fro} \lV\BRs\BSigmas^{-1/2}\rV_{2,\infty}    \lV\BR_{\BA_2}\rV_{\fro}\cr
\le &~ \frac{\eta^2}{(1-\varepsilon_0)^2}\sqrt{\frac{24c_s^2\mu^2r^2 \log n}{ m}} \lV\BDeltaLk\BSigmas^{1/2}\rV_{\fro}  \lV \BDeltaRk\BSigmas^{1/2}\rV_{\fro},
\end{align*}
\end{small}%
and
\begin{align*}
 &~\lv\l(1-\eta)\BDeltaLk\BSigmas^{1/2},\eta \left(\frac{1}{p}\G\Pi_\Omega\G^{*}-\G\G^{*}\right)\left(\BLn\BDeltaRk^*\right)\BRn\left(\BRn^{*}\BRn\right)^{-1}\BSigmas^{1/2}\r\rv\cr
  =&~\lv\l(1-\eta)\BDeltaLk\BSigmas\left(\BRn^{*}\BRn\right)^{-1}\BSigmas^{1/2} (\BRn\BSigmas^{-1/2})^{*},\eta \left(\frac{1}{p}\G\Pi_\Omega\G^{*}-\G\G^{*}\right)\left(\BLn\BDeltaRk^*\right)\r\rv\cr
 \le&~(1-\eta)\eta  \sqrt{\frac{24 n\log n}{p}}\lV\BL_{\BA_1}\rV_{\fro} \lV\BLn\BSigmas^{-1/2}\rV_{2,\infty} \lV\BRn\BSigmas^{-1/2}\rV_{2,\infty} \lV\BDeltaRk\BSigmas^{1/2}\rV_{\fro}\cr
 \le &~ \frac{(1-\eta)\eta}{(1-\varepsilon_0)^4}\sqrt{\frac{24c^4c_s^2\kappa^4\mu^2r^2 \log n}{ m}} \lV\BDeltaLk\BSigmas^{1/2}\rV_{\fro}  \lV \BDeltaRk\BSigmas^{1/2}\rV_{\fro}.
\end{align*}
By triangle inequality and \Cref{projectionerr1,projectionerr2} we have
\begin{small}
\begin{align*}
&~\lv\l\eta\BLs\BDeltaRk^{*}\BRn\left(\BRn^{*}\BRn\right)^{-1}\BSigmas^{1/2},\eta \left(\frac{1}{p}\G\Pi_\Omega\G^{*}-\G\G^{*}\right)\left(\BLn\BDeltaRk^{*}\right)\BRn\left(\BRn^{*}\BRn\right)^{-1}\BSigmas^{1/2}\r\rv\cr
\le&~\lv\l\eta \BLs\BSigmas^{-1/2}\BR_{\BA_2}^{*},\eta \left(\frac{1}{p}\G\Pi_\Omega\G^{*}-\G\G^{*}\right)\left(\BLs\BDeltaRk^{*}\right)\r\rv\cr
&~+\lv\l\eta(\BLs\BSigmas^{-1/2})\BR_{\BA_2}^{*},\eta \left(\frac{1}{p}\G\Pi_\Omega\G^{*}-\G\G^{*}\right)\left(\BDeltaLk\BDeltaRk^{*}\right)\r\rv\cr
\le &~ \varepsilon_0\eta^2\lV\BU_{\star}\BR_{\BA_2}^{*}\rV_{\fro} \lV\BLs\BDeltaRk^{*}\rV_{\fro}\cr
&+\eta^2\sqrt{\frac{24 n\log n}{p}} \lV\BLs\BSigmas^{-1/2}\rV_{2,\infty} \lV\BDeltaLk\BSigmas^{1/2}\rV_{\fro} \lV\BDeltaRk\BSigmas^{-1/2}\rV_{2,\infty}
 \lV \BR_{\BA_2} \rV_{\fro}\cr
\le &~\frac{\varepsilon_0\eta^2}{(1-\varepsilon_0)^2}\lV\BSigmas^{1/2} \BDeltaRk^{*}\rV_{\fro}^2
+\frac{\eta^2(1+c\kappa)}{(1-\varepsilon_0)^3}\sqrt{\frac{24c_s^2\mu^2r^r \log n}{m}} \lV\BDeltaLk\BSigmas^{1/2}\rV_{\fro} \lV \BDeltaRk\BSigmas^{1/2}\rV_{\fro},
\end{align*}
\end{small}%
where in the third inequality follows from $\lV\BU_{\star}\BR_{\BA_2}^{*}\rV_{\fro}\le\lV\BR_{\BA_2}\rV_{\fro}$, $\lV\BLs\BDeltaRk^{*}\rV_{\fro}=\lV\BU_{\star}\BSigmas^{1/2}\BDeltaRk^{*}\rV_{\fro}\le \lV\BSigmas^{1/2}\BDeltaRk^{*}\rV_{\fro}$ and \eqref{ineq:DeltaL2norm}, \eqref{ineq:DeltaLl2inf}. 
Then, combing all the pieces, for $c\ge 1+\varepsilon_0$, $\varepsilon_0,\eta\in(0,1)$,  provided $m\ge 24c^2(1+c\kappa)^2\varepsilon_0^{-2}c_s^2\mu^2r^2 \log n$, we then have
\begin{small}
\begin{align*}
&~\lv\l\bm{T}_1,\bm{T}_4\r\rv\le  \frac{(1-\eta)\eta \varepsilon_0}{1-\varepsilon_0}\lV\BDeltaLk\BSigmas^{1/2}\rV_{\fro}^2+ \frac{\varepsilon_0\eta^2}{(1-\varepsilon_0)^2}\lV\BSigmas^{1/2} \BDeltaRk^{*}\rV_{\fro}^2\cr
 &\qquad\qquad\qquad + \frac{2\varepsilon_0\eta(1+\varepsilon_0)}{(1-\varepsilon_0)^4}\lV\BDeltaLk\BSigmas^{1/2}\rV_{\fro}\lV\BDeltaRk\BSigmas^{1/2}\rV_{\fro}\cr
 \le  &\left(\frac{(1-\eta)\eta \varepsilon_0}{1-\varepsilon_0}+\frac{\varepsilon_0\eta(1+\varepsilon_0)}{(1-\varepsilon_0)^4}\right)\lV\BDeltaLk\BSigmas^{1/2}\rV_{\fro}^2+ \left(\frac{\varepsilon_0\eta^2}{(1-\varepsilon_0)^2}+\frac{\varepsilon_0\eta(1+\varepsilon_0)}{(1-\varepsilon_0)^4}\right)\lV\BSigmas^{1/2} \BDeltaRk^{*}\rV_{\fro}^2.
\end{align*}
\end{small}%
For $\lV \bm{T}_4\rV_{\fro}$: Using the variational representation of the Frobenius norm, for some $\BL_{\BA}\in \mathbb{C}^{n_1\times r}$, $\lV\BL_{\BA}\rV_{\fro}=1$ we have
\begin{small}
\begin{align*}
&~\lV\bm{T}_4\rV_{\fro}\cr
 =&~\eta \lv\l \left(\frac{1}{p}\G\Pi_\Omega\G^{*}-\G\G^{*}\right)\left(\BDeltaLk\BRs^{*}+\BLn\BDeltaRk^{*}\right)\BRn\left(\BRn^{*}\BRn\right)^{-1}\BSigmas^{1/2},\BL_{\BA}\r\rv\cr
 \le&~\eta \lv\l \left(\frac{1}{p}\G\Pi_\Omega\G^{*}-\G\G^{*}\right)\left(\BDeltaLk\BRs^{*}\right),\BL_{\BA}\BSigmas^{1/2}\left(\BRn^{*}\BRn\right)^{-1}\BRs^*\r\rv\cr
 &+\eta \lv\l \left(\frac{1}{p}\G\Pi_\Omega\G^{*}-\G\G^{*}\right)\left(\BDeltaLk\BRs^{*}\right),\BL_{\BA}\BSigmas^{1/2}\left(\BRn^{*}\BRn\right)^{-1}\BDeltaRk^*\r\rv\cr
 &+\eta \lv\l \left(\frac{1}{p}\G\Pi_\Omega\G^{*}-\G\G^{*}\right)\left(\BLn\BDeltaRk^{*}\right),\BL_{\BA}\BSigmas^{1/2}\left(\BRn^{*}\BRn\right)^{-1}\BRn^*\r\rv\cr
 \le&~ \varepsilon_0\eta\lV\BDeltaLk\BSigmas^{1/2}\BV_{\star}\rV_{\fro} \lV\BL_{\BA}\BSigmas^{1/2}\left(\BRn^{*}\BRn\right)^{-1}\BSigmas^{1/2}\BV_{\star}\rV_{\fro}\cr
 &+\eta\sqrt{\frac{24 n\log n}{p}} \Big(\lV\BDeltaLk\BSigmas^{1/2}\rV_{2,\infty} \lV\BL_{\BA}\rV_{\fro} \lV\BRs\BSigmas^{-1/2}\rV_{2,\infty}\lV \BDeltaRk\left(\BRn^{*}\BRn\right)^{-1}\BSigmas^{1/2}\rV_{\fro}\cr
 &+
 \lV\BLn\BSigmas^{-1/2}\rV_{2,\infty} \lV\BL_{\BA}\BSigmas^{1/2}\left(\BRn^{*}\BRn\right)^{-1}\BSigmas^{1/2}\rV_{\fro} \lV\BDeltaRk\BSigmas^{1/2}\rV_{\fro}\lV \BRs\BSigmas^{-1/2}\rV_{2,\infty}\Big) \cr
 \le&~ \frac{\varepsilon_0\eta}{(1-\varepsilon_0)^2}\lV\BDeltaLk\BSigmas^{1/2}\rV_{\fro}+\frac{2\eta(1+c\kappa)}{(1-\varepsilon_0)^3}\sqrt{\frac{24c_s^2\mu^2r^2 \log n}{m}}\lV \BDeltaRk\BSigmas^{1/2}\rV_{\fro}
\end{align*}
\end{small}%
where the second inequality follows from \Cref{projectionerr1,projectionerr2}, and the last inequality follows from \eqref{boundnorm:c}, \eqref{boundnorm:d}, \eqref{ineq:DeltaLl2inf}, \eqref{ineq:DeltaL2norm}, as well as the following inequalities: $\|\BL_{\BA}\BSigmas^{1/2}\left(\BRn^{*}\BRn\right)^{-1}\BSigmas^{1/2}\BV_{\star}^*\|_{\fro}\le \|\BL_{\BA}\BSigmas^{1/2}\left(\BRn^{*}\BRn\right)^{-1}\BSigmas^{1/2}\|_{\fro}\le \frac{1}{(1-\varepsilon_0)^2}$ and $\| \BDeltaRk\left(\BRn^{*}\BRn\right)^{-1}\BSigmas^{1/2}\|_{\fro}\le \frac{1}{(1-\varepsilon_0)^2}\| \BDeltaRk\BSigmas^{-1/2}\|_{\fro}\le \frac{1}{\sigmas_r(1-\varepsilon_0)^2}\| \BDeltaRk\BSigmas^{1/2}\|_{\fro}$.
\end{proof}

\section{Conclusion} 
In this paper, we introduce a novel non-convex algorithm, dubbed Hankel Structured Newton-Like Descent, for accelerating ill-conditioned robust Hankel recovery problems. The linear convergence guarantee has been established for the proposed algorithm; especially, its convergence rate is independent of the condition number of the underlying Hankel matrix. The empirical advantages of the proposed algorithm have been verified via extensive experiments on synthetic and real datasets.

{
\bibliographystyle{plain}
\bibliography{robustHMR}
}

\end{document}